\documentclass{article}
\usepackage[final,nonatbib]{neurips_2021}
\usepackage{graphicx,epstopdf}
\usepackage{epsfig}
\usepackage{xcolor}
\usepackage{algorithm}
\usepackage[noend]{algpseudocode}

\def\cmtinclude{0}
\if\cmtinclude1
\newcommand{\SO}[1]{\textcolor{darkgreen}{SO: #1}}
\newcommand{\so}[1]{\textcolor{darkblue}{#1}}
\newcommand{\clr}[1]{\textcolor{clor}{#1}}
\newcommand{\an}[1]{\textcolor{red}{AN: #1}}
\newcommand{\mf}[1]{{\color{darkred} MF: #1}}
\newcommand{\ys}[1]{{\color{blue} YS: #1}}
\newcommand{\ig}[1]{{\color{brown} IG: #1}}
\newcommand{\new}[1]{\textcolor{orange}{#1}}
\newcommand{\past}[1]{{\color{olive} past: #1}}

\newcommand{\annew}[1]{{\color{red} AN new: #1}}

\newcommand{\red}[1]{\textcolor{red}{#1}}
\newcommand{\dred}[1]{\textcolor{black}{\emph{#1}}}
\else
\newcommand{\clr}[1]{}
\newcommand{\ig}[1]{}
\newcommand{\so}[1]{}
\newcommand{\an}[1]{}

\newcommand{\annew}[1]{#1}
\newcommand{\SO}[1]{}
\newcommand{\mf}[1]{}
\newcommand{\ys}[1]{}
\newcommand{\new}[1]{}
\newcommand{\past}[1]{}

\newcommand{\red}[1]{#1}
\newcommand{\dred}[1]{\emph{#1}}
\fi

\usepackage{microtype}
\usepackage{enumitem}
\setlist[itemize]{noitemsep, nolistsep}
\setlist[enumerate]{noitemsep, nolistsep}
\usepackage{multicol}
\usepackage{float}
\usepackage{wrapfig}
\usepackage{graphicx}
\usepackage{subfigure}
\usepackage{caption}
\usepackage{mathrsfs}
\usepackage{thm-restate}
\usepackage{hyperref,graphicx,amsmath,amssymb,bm,breakurl,epsfig,epsf,color,mathbbol,fmtcount,semtrans,multirow,comment,boldline,movie15,pgfplots}
\usepackage{cellspace}
\usepackage{makecell}
\setcellgapes{5pt}
\usepackage{booktabs} 
\usepackage{hyperref}

\usepackage{pifont}

\usepackage{tikz}
\usetikzlibrary{pgfplots.groupplots}
\usepackage[utf8]{inputenc} 
\usepackage[T1]{fontenc}    
\usepackage{url}            
\usepackage{booktabs}       
\usepackage{amsfonts}       
\usepackage{nicefrac}       
\usepackage{microtype}      
\usepackage{mathtools}
\definecolor{darkred}{RGB}{150,0,0}
\definecolor{darkgreen}{RGB}{0,100,0}
\definecolor{darkblue}{RGB}{0,0,200}
\definecolor{clor}{RGB}{0,100,100}

\newtheorem{theorem}{Theorem}

\newtheorem{assumption}{Assumption}

\newtheorem{lemma}{Lemma}
\newtheorem{corollary}{Corollary}

\newtheorem{definition}{Definition}

\newtheorem{observation}{Observation}

\newtheorem{remark}{Remark}

\newtheorem{example}{Example}
\numberwithin{equation}{section}

\usepackage{etoolbox}

\newcommand{\citep}{\cite}

\def \endprf{\hfill {\vrule height6pt width6pt depth0pt}\medskip}

\newenvironment{proof}{\noindent {\bf Proof} }{\endprf\par}

\newcommand{\tsn}[1]{{\left\vert\kern-0.25ex\left\vert\kern-0.25ex\left\vert #1 
    \right\vert\kern-0.25ex\right\vert\kern-0.25ex\right\vert}}

\newcommand\tr{{\bf tr}}
\newcommand{\Dct}{\Dc_T}
\newcommand{\Dcf}{\Dc_F}
\newcommand{\eps}{\varepsilon}

\newcommand{\st}{\star}
\newcommand{\distas}{\overset{\text{i.i.d.}}{\sim}}

\newcommand{\beq}{\begin{equation}}
\newcommand{\ba}{\begin{align}}
\newcommand{\ea}{\end{align}}

\newcommand{\eeq}{\end{equation}}

\newcommand{\nn}{\nonumber}
\newcommand{\la}{\lambda}

\newcommand{\A}{{\mtx{A}}}

\newcommand{\Ub}{{\mtx{U}}}

\newcommand{\B}{{\mtx{\Sigma}_T}}
\newcommand{\Bi}{{\mtx{\Sigma}_{T,i}}}
\newcommand{\Bbar}{{\tilde{\mtx{\Sigma}}_{T}}}
\newcommand{\Sbar}{{\tilde{\mtx{S}}_{T}}}
\newcommand{\BbarR}{{\tilde{\mtx{\Sigma}}_{T}^{R}}}
\newcommand{\BbarRi}{{\tilde{\mtx{\Sigma}}^{R}_{T,i}}}

\newcommand{\Bsq}{{\mtx{\Sigma}^2_{T}}}
\newcommand{\hB}{{\hat{\mtx{\Sigma}}_{T}}}

\newcommand{\Bp}{{{\mtx{\Sigma}}_{T}'}}
\newcommand{\barB}{{\bar{\mtx{\Sigma}}_{T}}}

\newcommand{\Ib}{{{\mtx{I}}}}
\newcommand{\risk}{\text{risk}}

\newcommand{\diag}[1]{\text{diag}(#1)}

\newcommand{\Dc}{{\cal{D}}}

\newcommand{\Qb}{{\mtx{Q}}}

\newcommand{\Rr}{R}

\newcommand{\Eb}{{\mtx{E}}}

\newcommand{\bSi}{{\boldsymbol{{\Sigma}}_F}}
\newcommand{\bSii}{{\boldsymbol{{\Sigma}}_{F,i}}}

\newcommand{\bSiR}{{\boldsymbol{{\Sigma}}_{F}^{R}}}
\newcommand{\bSt}{{\boldsymbol{{\Sigma}}_{\bt}}}
\newcommand{\bSisq}{{\boldsymbol{{\Sigma}}^2_{F}}}
\newcommand{\hbSi}{{\hat{\mtx{\Sigma}}_{F}}}

\newcommand{\iotaT}{\iota_T}
\newcommand{\iotaF}{\iota_F}
\newcommand{\Sigt}{\bSigma_T}
\newcommand{\Sigf}{\bSigma_F}

\newcommand{\Iden}{{\mtx{I}}}
\newcommand{\M}{{\mtx{M}}}

\newcommand{\order}[1]{{\cal{O}}(#1)}
\newcommand{\Order}{{\cal{O}}}

\newcommand{\z}{{\vct{z}}}

\newcommand{\tn}[1]{\|{#1}\|_{2}}

\newcommand{\bt}{{\boldsymbol{\beta}}}

\newcommand{\bts}{{\boldsymbol{\beta}_\st}}

\newcommand{\btheta}{{\boldsymbol{\theta}}}

\newcommand{\ddelta}{{\boldsymbol{\delta}}}

\newcommand{\Nn}{\mathcal{N}}

\newcommand{\vb}{\vct{v}}

\newcommand{\bb}{\vct{b}}
\newcommand{\ub}{{\vct{u}}}

\newcommand{\h}{\vct{h}}

\newcommand{\Zb}{\mtx{Z}}

\renewcommand{\d}{\mathrm{d}}

\newcommand{\x}{\vct{x}}

\newcommand{\y}{\vct{y}}

\newcommand{\W}{\mtx{W}}

\definecolor{emmanuel}{RGB}{255,127,0}

\newcommand{\R}{\mathbb{R}}

\newcommand{\E}{\operatorname{\mathbb{E}}}

\newcommand{\DoF}[1]{DoF}

\newcommand{\vct}[1]{\bm{#1}}
\newcommand{\mtx}[1]{\bm{#1}}

\newcommand{\Id}{\text{\em I}}

\newcommand{\X}{{\mtx{X}}}

\newcommand{\Vb}{{\mtx{V}}}

\newcommand{\sige}{\sigma}

\newcommand{\ones}{\mathbf{1}}
\newcommand{\cmt}[1]{}

\newcommand{\argmin}{\mathrm{argmin}}

\newcommand{\xij}{\x_{ij}}
\newcommand{\yij}{y_{ij}}

\newcommand{\sbSi}{{\boldsymbol{{\Sigma}}^{1/2}_{F}}}
\newcommand{\sbSiinv}{{\boldsymbol{{\Sigma}}^{-1/2}_{F}}}
\def\bM{\mtx{M}}
\def\bMt{{\mtx{M}}_{\text{top}}}

\newcommand{\rP}{\stackrel{{P}}{\longrightarrow}}
\newcommand{\bg}{\boldsymbol{\phi}}

\newcommand{\mfkb}{{\mtx{B}}}
\newcommand{\uth}{{\underline{\theta}}}

\newcommand{\calE}{{\mathcal{E}}}

\newcommand{\ntask}{T}
\newcommand{\ntot}{N}
\newcommand{\nspt}{n_1}
\newcommand{\nfs}{n_2}

\newcommand{\bLa}{\hbLa^*}
\newcommand{\hbLa}{\boldsymbol{\Lambda}}
\newcommand{\hbLaRi}{\hbLa_{R,i}}

\newcommand{\rf}{s_F}
\newcommand{\rt}{s_T}
\newcommand{\barrt}{r_T}
\newcommand{\barrf}{r_F}

\newcommand{\TX}{\barrf}
\newcommand{\Tbt}{\barrt}
\newcommand{\TbtX}{\tilde{r}_{T}}

\newcommand{\Snorm}{\mathcal{S}}
\newcommand{\Xasym}{{\tilde \X}}
\newcommand{\yasym}{{\tilde \y}}
\newcommand{\bSiasym}{{{\tilde{\mtx{\Sigma}}_{\X}}}}
\newcommand{\bSiasymi}{{\tilde{\mtx{\Sigma}}_{{\X}_{i,i}}}}
\newcommand{\colsmfkb}{{\max\{\ntask \|\bSi\|,d\lambda_{\rf+1}(\bSi)\}}}
\newcommand{\bSigma}{{\boldsymbol{{\Sigma}}}}

\newcommand{\BB}{\mathbf{G}}
\newcommand{\bBB}{\bar{\mathbf{G}}}

\newcommand{\sigmaze}{\sigma_R}
\newcommand{\hatbMt}{{\hat \bM}_{\text{top}}}

\makeatletter
\renewcommand{\Function}[2]{%
  \csname ALG@cmd@\ALG@L @Function\endcsname{#1}{#2}%
  \def\jayden@currentfunction{#1}%
}

\newcommand{\funclabel}[1]{%
  \@bsphack
  \protected@write\@auxout{}{%
    \string\newlabel{#1}{{\jayden@currentfunction}{\thepage}}%
  }%
  \@esphack
}

\title{Towards Sample-efficient Overparameterized Meta-learning}

\author{
Yue Sun \\
University of Washington \\
\texttt{yuesun@uw.edu}\\
\And
Adhyyan Narang \\
University of Washington\\ 
\texttt{adhyyan@uw.edu}
\And
Halil Ibrahim Gulluk \\
Bogazici University \\
\texttt{hibrahimgulluk@gmail.com}
\And
Samet Oymak \\
University of California, Riverside \\ 
\texttt{oymak@ece.ucr.edu}
\And
Maryam Fazel\\
University of Washington \\
\texttt{mfazel@uw.edu}
}

\begin{document}
\maketitle

\begin{abstract}
An overarching goal in machine learning is to build a generalizable model with few samples. To this end, overparameterization has been the subject of immense interest to explain the generalization ability of deep nets even when the size of the dataset is smaller than that of the model. While the prior literature focuses on the classical supervised setting, this paper aims to demystify overparameterization for meta-learning. Here we have a sequence of linear-regression tasks and we ask: (1) Given earlier tasks, what is the optimal linear representation of features for a new downstream task? and (2) How many samples do we need to build this representation? This work shows that surprisingly, overparameterization arises as a natural answer to these fundamental meta-learning questions. Specifically, for (1), we first show that learning the optimal representation coincides with the problem of designing a task-aware regularization to promote inductive bias. We leverage this inductive bias to explain how the downstream task actually benefits from overparameterization, in contrast to prior works on few-shot learning. For (2), we develop a theory to explain how feature covariance can implicitly help reduce the sample complexity well below the degrees of freedom and lead to small estimation error. We then integrate these findings to obtain an overall performance guarantee for our meta-learning algorithm. Numerical experiments on real and synthetic data verify our insights on overparameterized meta-learning.

\end{abstract}

\tableofcontents

\section{Introduction}\label{sec:intro}

In a multitude of machine learning (ML) tasks with limited data, it is crucial to build accurate models in a sample-efficient way. Constructing a simple yet informative representation of features is a critical component of learning a model that generalizes well to an unseen test set. The field of meta-learning dates back to \cite{caruana1997multitask,baxter2000model} and addresses this challenge by transferring insights across distinct but related tasks. Usually, the meta-learner first (1) learns a feature-representation from previously seen tasks and then (2) uses this representation to succeed at an unseen task. The first phase is called representation learning and the second is called few-shot learning. Such information transfer between tasks is the backbone of modern transfer and multitask learning and finds ubiquitous applications in image classification \citep{deng2009imagenet}, machine translation \citep{bojar2014findings} and reinforcement learning \citep{finn2017model}.

Recent literature in ML theory has posited that overparameterization can be beneficial to generalization in traditional single-task setups for both regression \cite{mei2019generalization, wu2020optimal, bartlett2020benign, muthukumar2019harmless, montanari2019generalization} and classification \cite{muthukumar2020classification, montanari2020generalization} problems. Empirical literature in deep learning suggests that overparameterization is of interest for both phases of meta-learning as well. Deep networks are stellar representation learners despite containing many more parameters than the sample size. 
Additionally, overparameterization is observed to be beneficial in the few-shot phase for transfer-learning in Figure~\ref{fig:fs_double_descent}(a). A ResNet-50 network pretrained on Imagenet was utilized to obtain a representation of $R$ features for classification on CIFAR-10. 
All layers except the final (softmax) layer are frozen and are treated as a fixed feature-map. We then train the final layer of the network for the downstream task which yields a linear classifier on pretrained features. The figure plots the effect of increasing $R$ on the test error on CIFAR-10, for different choices of training size $\nfs$. For each choice of $\nfs$, increasing $R$ beyond $\nfs$ is seen to reduce the test-error. These findings are corroborated by \cite{finn2017model} (MAML) and \cite{vinyals2016matching}, who successfully use a transfer learning method that adapts a pre-trained model, with $112980$ parameters, to downstream tasks with only 1-5 new training samples.

\setlength{\textfloatsep}{.2em}
\begin{figure}[htbp!]
\centering
\subfigure{
\includegraphics[width =0.43\linewidth, height=0.32\linewidth]{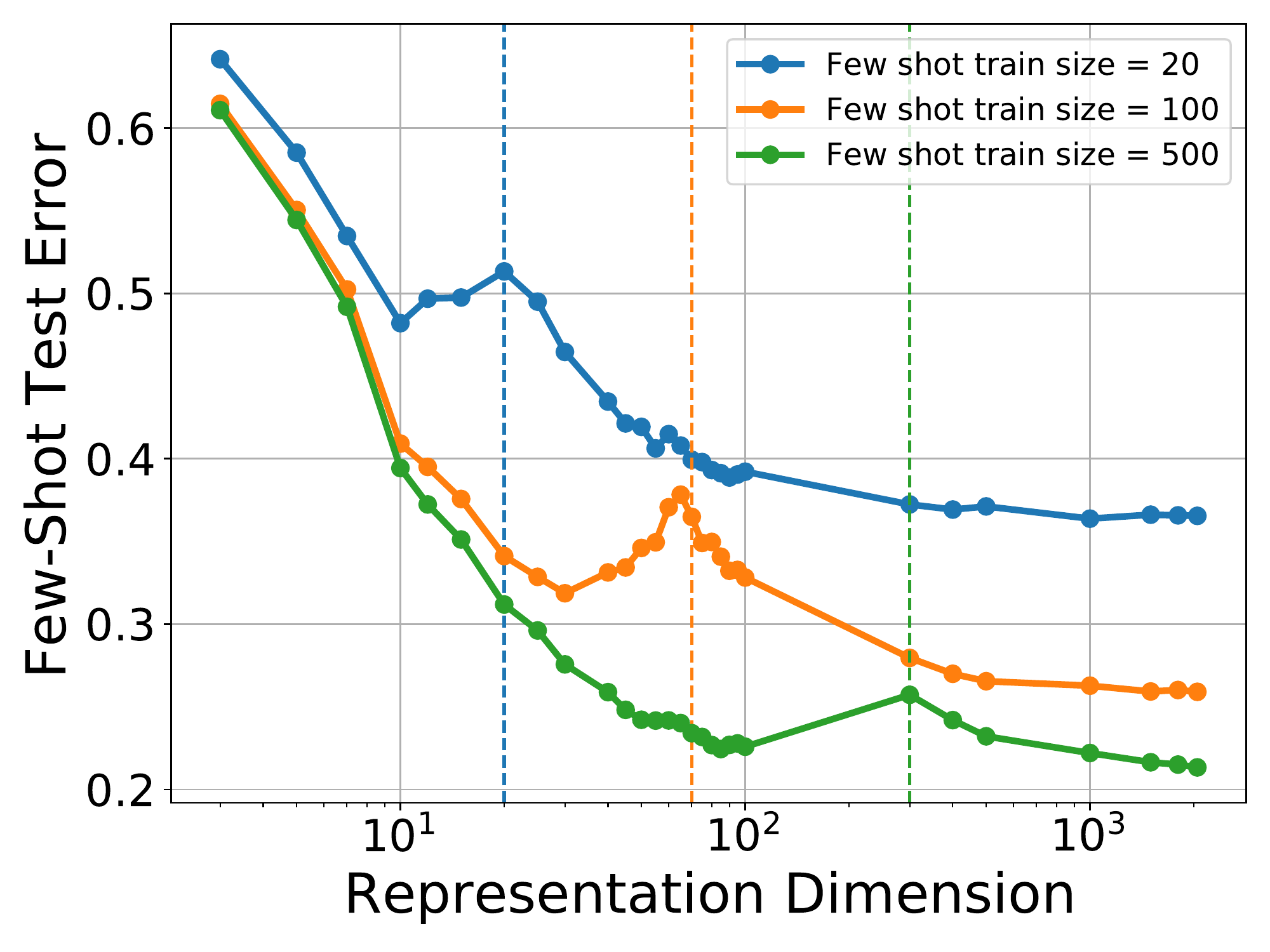}}
\subfigure{
\includegraphics[width =0.43\linewidth,height=0.32\linewidth]{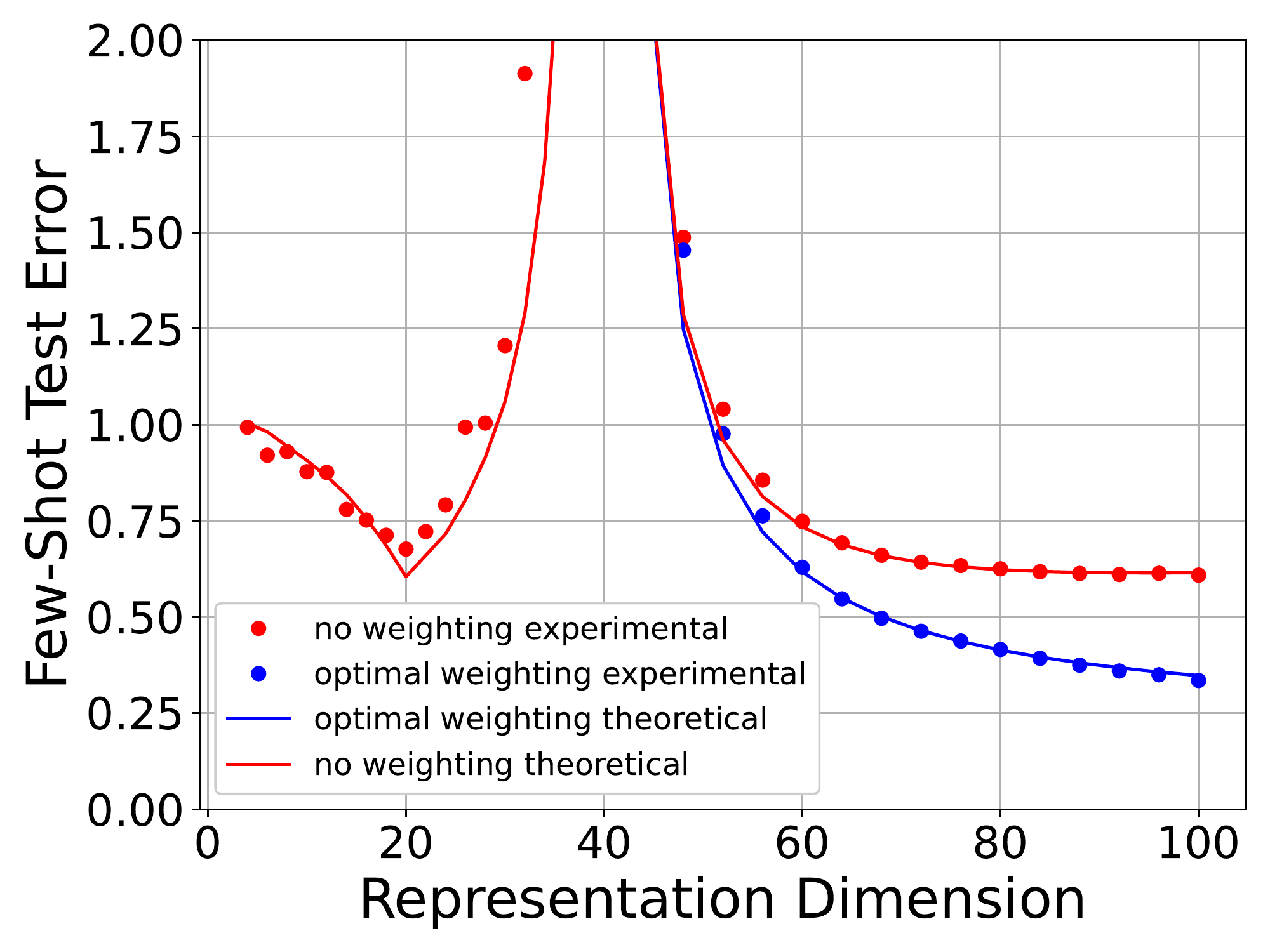}
}
\vspace{-4pt}\caption{\textbf{Illustration of the benefit of overparameterization in the few-shot phase.}
(a) Double-descent in transfer learning: dashed lines indicate the location where the number of features $R$ exceed the number of training points; i.e., the transition from under to over-parameterization. The experimental details are contained in the supplement. (b)
Illustration of the benefit of using Weighted minL2-interpolation in Definition~\ref{def:minl2_interpolator} (blue). 
See Remark \ref{remark: fig 1b} for details and discussion. 
} 
\label{fig:fs_double_descent}\vspace{7pt}
\end{figure}

In Figure~\ref{fig:fs_double_descent}(b), we consider a sequence of \emph{linear} regression tasks and plot the few-shot error of our proposed projection and eigen-weighting based meta-learning algorithm for a fixed few-shot training size, but varying dimensionality of features. The resulting curve 
looks similar to Figure~\ref{fig:fs_double_descent}(a) and suggests that the observations regarding overparameterization for meta-learning in neural networks can,  
to a good extent, be captured by linear models, thus motivating their detailed study. This aligns with trends in recent literature: while deep nets are nonlinear, recent advances show that linearized problems such as kernel regression (e.g., via neural tangent kernel \cite{jacot2018neural,du2018gradient,lee2019wide,oymak2019generalization,chizat2018lazy}) provide a good proxy to understand some of the theoretical properties of practical overparameterized  deep nets.

However, existing analysis of subspace-based meta-learning algorithms for both the representation learning and few-shot phases of linear models have typically focused on the classical \emph{underparameterized regime}. These works (see Paragraphs 2-3 of Sec. \ref{sec:related_work}) consider the case where representation learning involves projection onto a lower-dimensional subspace. On the other hand, recent works on double descent shows that an  \emph{overparameterized} interpolator beats PCA-based method. This motivates us to build upon these results to develop a theoretical understanding of overparameterized meta-learning\footnote{The code for this paper is in \url{https://github.com/sunyue93/Rep-Learning}.}.

\vspace{-.5em}
\subsection{Our contributions}
\vspace{-.5em}
This paper studies meta-learning when each task is a linear regression problem, similar in spirit to \cite{tripuraneni2020provable,kong2020meta}. In the representation learning phase, the learner is provided with training data from $\ntask$ distinct tasks, with $\nspt$ training samples per task: using this data, it selects a matrix $\hbLa\in\R^{d\times R}$ with arbitrary $R$ to obtain a linear \emph{representation} of features via the map  $\x \to \hbLa^\top \x$. In the few-shot learning phase, the learner faces a new task with $n_2$ training samples and aims to use the representation $\hbLa^\top \x$ to aid prediction performance.

We highlight that obtaining the representation consists of two steps: first the learner projects $\x$ onto $R$ basis directions, and then performs \emph{eigen-weighting} of each of these directions, as shown in Figure~\ref{subfig:eigenweighting}. 
The overarching goal of this paper is to propose a scheme to use the knowledge gained from earlier tasks to choose $\hbLa$ that minimizes few-shot risk. This goal enables us to engage with important questions regarding overparameterization: 
\vspace*{-3pt}
\patchcmd{\quote}{\rightmargin}{\leftmargin 1em \rightmargin}{}{}
\begin{quote}
\textbf{Q1:} What should the size $R$ and the representation $\hbLa$ be to minimize risk at the few-shot phase?\vspace{2pt}

\textbf{Q2:} Can we learn the $Rd$ dimensional representation $\hbLa$ with $\ntot\ll \Rr d$ samples?
\end{quote}\vspace*{-3pt}

The answers to the questions above will shed light on whether overparameterization is beneficial in few-shot learning and representation learning respectively. Towards this goal, we make several contributions to the finite-sample understanding of \emph{linear} meta-learning, under assumptions discussed in Section \ref{sec:setup}. Our results are obtained for a general data/task model with \emph{arbitrary task covariance $\bSt$ and feature covariance $\bSi$} which allows for a rich set of observations.

\textbf{Optimal representation for few-shot learning.} As a stepping stone towards the goal of characterizing few-shot risk for different $\hbLa$, in Section~\ref{sec:opt_rep} we first consider learning with \textbf{known covariances} $\B$ and $\bSi$ respectively (Algorithm \ref{algo:opt rep}). Compared to projection-only representations in previous works (see Paragraphs 2-3 of Sec. \ref{sec:related_work}), our scheme applies \emph{eigen-weighting} matrix $\bLa$ to incentivize the optimizer to place higher weight on promising eigen-directions. This eigen-weighting procedure has been shown in the single-task case to be extremely crucial to avail the benefit of overparameterization \cite{belkin2019reconciling, montanari2019generalization, muthukumar2019harmless}: it captures an inductive bias that promotes certain features and demotes others. We show that the importance of eigen-weighting extends to the multi-task case as well.

\textbf{Canonical task covariance.} Our analysis in Section~\ref{sec:opt_rep} also reveals that, the optimal subspace and representation matrix are closed-form functions of the \emph{canonical task covariance} $\Bbar = \sbSi\B\sbSi$, which captures the feature saliency by summarizing the feature and task distributions.

\begin{wraptable}{r}{.45\textwidth}
\centering
\begin{tabular}{|Sc|c|}
\hline
$\bSi$ & Feature covariance\\
\hline
$\B$ & Task covariance\\
\hline 
$\Bbar$ & Canonical task covariance\\ 
\hline
$\nspt$ & Samples per each earlier task \\
\hline
$\ntask$ & Number of earlier tasks\\
\hline
$\ntot$ & Total sample size $T\times n_1$\\
\hline
$\nfs$ & Samples for new task \\
\hline
$\hbLa$ & Eigen-weighting matrix \\
\hline 
\end{tabular}
\caption{Main notation}
\end{wraptable} 

\textbf{Representation learning.} In practice, task and feature covariances (and hence the canonical covariance) are rarely known apriori. However, we can estimate the principal subspace of the canonical task covariance $\Bbar$ (which has a degree of freedom (DoF) of $\Omega(Rd)$) from data. In Section~\ref{sec:meta train} we first present empirical evidence that feature covariance $\bSi$ is ``positively correlated'' with $\Bbar$.    Then we propose an efficient algorithm based on Method-of-Moments (MoM), and show that the sample complexity of representation learning is well below $\order{Rd}$ due to the inductive bias. Our sample complexity bound depends on interpretable quantities such as \emph{effective ranks} $\bSi,\Bbar$ and improves over prior art (e.g.,~\cite{kong2020meta,tripuraneni2020provable}), even though the prior works were specialized to low-rank $\Bbar$ and identity $\bSi$ (see Table \ref{table:1}).

\textbf{End to end meta-learning guarantee.} In Section~\ref{sec:robust}, we consider the generalization of Section~\ref{sec:opt_rep}, where we have only estimates of the covariances instead of perfect knowledge. This leads to an overall meta-learning guarantee in terms of $\bLa$, $\ntot$ and $\nfs$ and uncovers a bias-variance tradeoff: As $\ntot$ decreases, it becomes more preferable to use a smaller $\Rr$ (more bias, less variance) due to inaccurate estimate of the weak eigen-directions of $\Bbar$. In other words, we find that overparameterization is only beneficial for few-shot learning if the quality of representation learning is sufficiently good. This explains why, in practice, increasing the representation dimension may not help reduce few-shot risk beyond a certain point (see Fig. \ref{fig2}).

\begin{figure}[t!]
\centering
\subfigure{
\label{subfig:notation}
\includegraphics[width =0.58\textwidth]{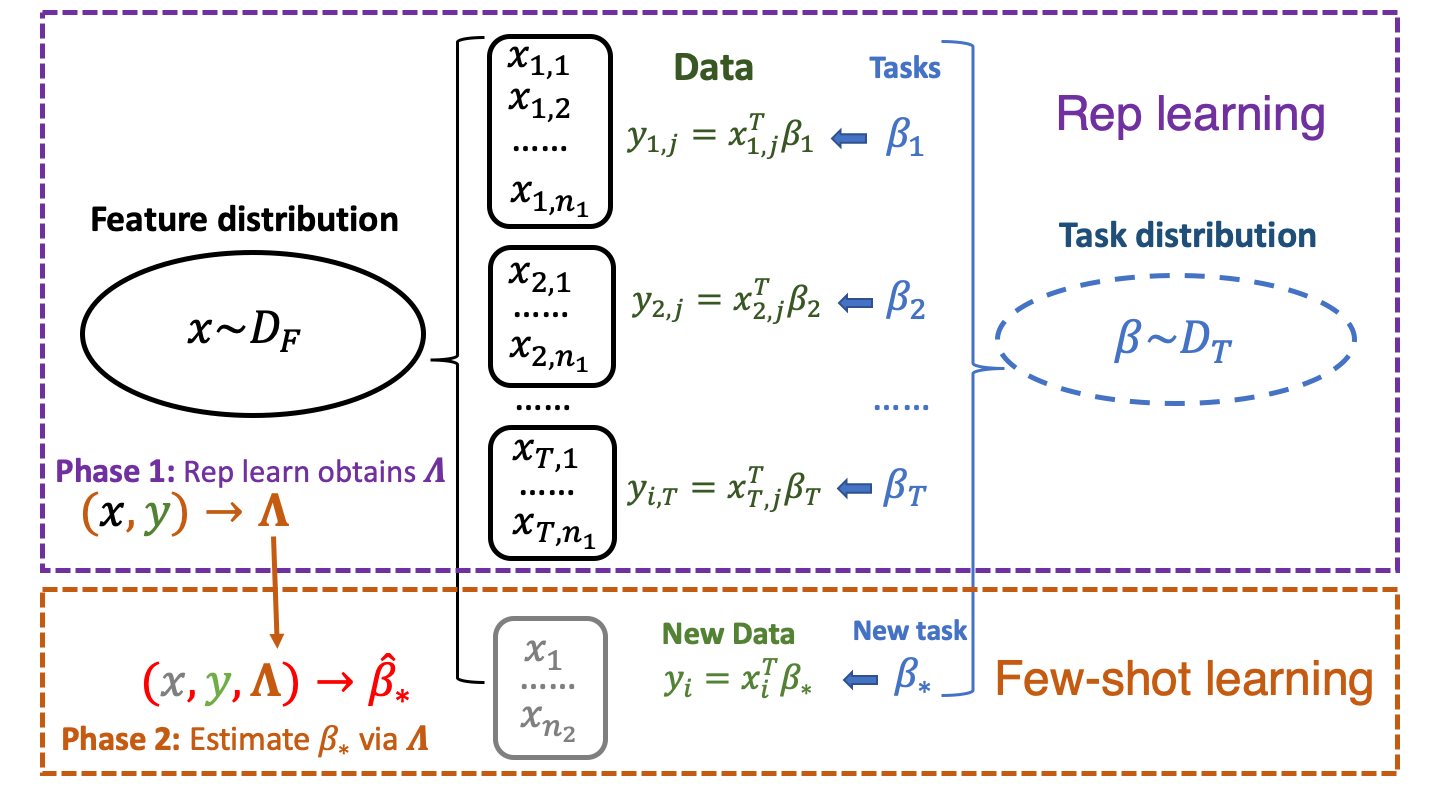}
}
\hspace{-.5em}
\subfigure{
\label{subfig:eigenweighting}
\includegraphics[width =0.35\textwidth]{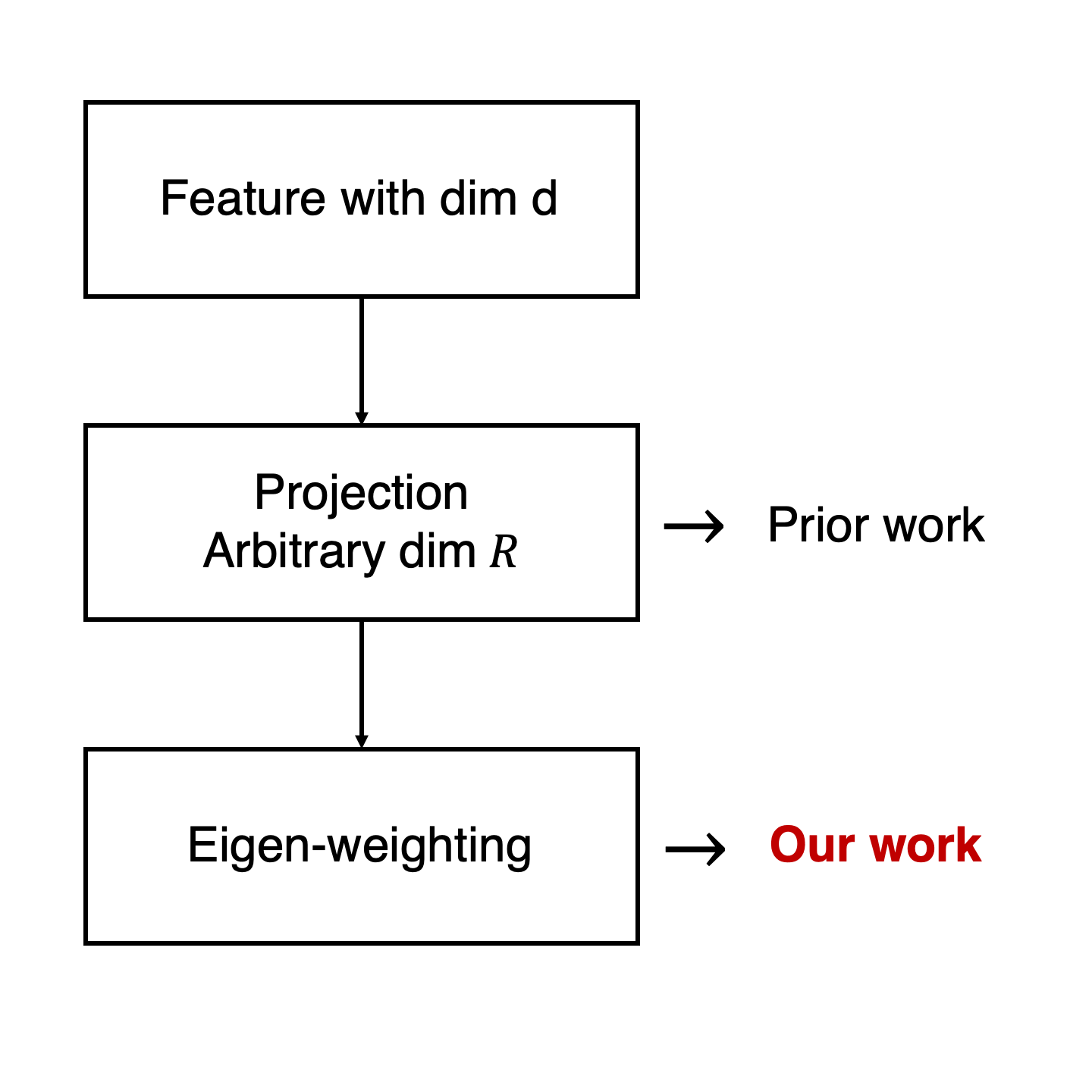}
}
\caption{(a) Steps of the meta-learning algorithm. (b) 
Our representation-learning algorithm has two steps:  projection and eigen-weighting. We focus on the use of overparameterization$+$weighting matrix (Def. \ref{def:minl2_interpolator}), and compare this with overparameterization with simple projection (no eigen-weighting), and underparameterization (\red{for which eigen-weighting has no impact and is equivalent to projection}).  \cite{tripuraneni2020provable,kong2020meta,kong2020robust,du2020few} study underparameterized projections only. To distinguish from eigen-weighting, we will refer to simple projections as subspace-based representations.
\\} \label{fig_meta_learn}
\end{figure}

\vspace{-.5em}
\subsection{Related work}
\label{sec:related_work}

\noindent\textbf{Overparameterized ML and double-descent}
The phenomenon of double-descent was first discovered by \cite{belkin2019reconciling}. This paper and subsequent works on this topic \cite{bartlett2020benign, muthukumar2019harmless, muthukumar2020classification, montanari2019generalization, chang2020provable} emphasize the importance of the right prior (sometimes referred to as inductive bias or regularization) to avail the benefits of overparameterization. However, an important question that arises is: where does this prior come from? Our work shows that the prior can come from the insights learned from related previously-seen tasks. Section~\ref{sec:opt_rep} extends the ideas in \cite{nakkiran2020optimal,wu2020optimal} to depict how the optimal representation described can be learned from imperfect covariance estimates as well.

\noindent\textbf{Theory for representation learning} 
Recent papers \citep{kong2020meta,kong2020robust,tripuraneni2020provable, du2020few} propose the theoretical bounds of representation learning when the tasks lie in an exactly $r$ dimensional subspace.  \citep{kong2020meta,kong2020robust,tripuraneni2020provable} discuss method of moment estimators and \citep{tripuraneni2020provable, du2020few} discuss matrix factorized formulations.
\cite{tripuraneni2020provable} shows that
the number of samples that enable meaningful representation learning is $\order{dr^2}$. \cite{kong2020meta,kong2020robust,tripuraneni2020provable} assume the features follow a standard normal distribution. We define a canonical covariance which handles arbitrary feature and task covariances. We also show that our estimator succeeds with $\order{dr}$ samples when $\nspt\sim r$, and extend the bound to general covariances with effective rank defined.

\noindent\textbf{Subspace-based meta learning} With tasks being low rank, \cite{kong2020meta,kong2020robust,tripuraneni2020provable,gulluk2021sample, du2020few} do few-shot learning in a low dimensional space.
 \cite{yang2020provable,yang2021impact} study 
meta-learning for linear bandits. \cite{lucas2020theoretical} 
gives information theoretic lower and upper bounds. 
\cite{bouniot2020towards} proposes subspace-based methods for nonlinear problems such as classification. We investigate a representation with arbitrary dimension, specifically interested in overparameterized case and show it yields a smaller error with general task/feature covariances. Related work \cite{du2020few} provides results on overparameterized representation learning, but \cite{du2020few} requires number of samples per pre-training task to obey $\nspt\gtrsim d$, whereas our results apply as soon as $\nspt\gtrsim 1$.

\noindent\textbf{Mixed Linear Regression (MLR)} In MLR  \citep{zhong2016mixed, li2018learning, chen2020learning}, multiple linear regression are executed, similar to representation learning. The difference is that, the tasks are drawn from a finite set, and number of tasks can be larger than $d$ and not necessarily low rank. \cite{lounici2011oracle,cavallanti2010linear,maurer2016benefit} propose sample complexity bounds of representation learning for mixed linear regression. They can be combined with other structures such as binary task vectors  \cite{balcan2015efficient} and sparse task vectors \citep{argyriou2008convex}.

\vspace{-.5em}
\section{Problem Setup}\label{sec:setup}

\vspace{-.5em}

The problem we consider consists of two phases: 
\begin{enumerate}[leftmargin=*]
    \item Representation learning: Prior tasks are used to learn a suitable representation to process features.
    
    \vspace{.2em}
    \item Few-shot learning: A new task is learned with a few samples by using the suitable representation.
\end{enumerate}

This section defines the key notations and describes the data generation procedure for the two phases. In summary, we study linear regression tasks, the features and tasks are generated randomly, i.i.d. from their associated distributions $\Dct$ and $\Dcf$, and the two phases share the same feature and task distributions.The setup is summarized in Figure~\ref{fig_meta_learn}(a).

\vspace{-.5em}
\subsection{Data generation}

\begin{definition}[Task and feature distributions]
\label{def:task_feature_dist}
Throughout, $\Dct$ and $\Dcf$ denote the distributions of tasks $\bt_i$ and features $\x_{ij}$ respectively. These distributions are subGaussian, zero-mean with corresponding covariance matrices $\Sigt$ and $\Sigf$. 
\end{definition}
\vspace{-.5em}

\vspace{-.5em}
\begin{definition}[Data distribution for a single task]
\label{def:data_single_task}
Given a specific realization of task vector $\bt \sim \Dct$, the corresponding label/input distribution $(y,\x)\sim \Dc_{\bt}$ is obtained via $y=\x^\top\bt+\eps$ where $\x\sim\Dcf$ and $\eps$ is zero-mean subgaussian noise with variance $\sige^2$.
\end{definition}
\vspace{-.5em}

\noindent\textbf{Data for Representation Learning (Phase 1).} We have $\ntask$ tasks, each with $\nspt$ training examples. The task vectors $(\bt_i)_{i=1}^\ntask\subset\R^d$ are drawn i.i.d.~from the distribution $\Dct$. The data for $i$th task is given by $(y_{ij},\x_{ij})_{j=1}^{n_1}\distas \Dc_{\bt_i}$. In total, there are $\ntot=\ntask\times \nspt$ examples.

\noindent\textbf{Data for Few-Shot Learning  (Phase 2).}  Sample task $\bts\sim\Dct$. Few-shot dataset has $\nfs$ examples $(y_{i},\x_{i})_{j=1}^{\nfs}\distas \Dc_{\bts}$.

We use representation learning data to learn a representation of feature-task distribution, called eigen-weighting matrix $\hbLa$ in Def. \ref{def:minl2_interpolator} below. The matrix $\hbLa$ is passed to few-shot learning stage, helping learn $\bts$ with few data.

\subsection{Training in Phase 2}

We will define a weighted representation, called eigen-weighting matrix, and show how it is applied for few-shot learning. The matrix is learned during representation learning using the data from the $\ntask$ tasks.

Denote $\X\in\R^{\nfs\times d}$ whose $i^{\rm th}$ row is $\x_i$, and $\y = [y_1,...,y_m]^\top$. We are interested in studying the weighted 2-norm interpolator defined below for overparameterization regime $R\ge\nfs$.

\begin{definition}[Eigen-weighting matrix and Weighted $\ell_2$-norm interpolator] \label{def:minl2_interpolator}
Let the representation dimension be $R$, where $R$ is any integer between $1$ and $d$. We define an eigen-weighting matrix $\hbLa\in\R^{d\times R}$ and the associated weighted $\ell_2$-norm interpolator
\begin{align*}
   \hat \bt_{\hbLa} = \arg\min_{\bt} \tn{\hbLa^{\dag} \bt} \quad \text{s.t.}\quad\y=\X \bt~~~\text{and}~~~ \bt\in\mathrm{range\_space}(\hbLa).
\end{align*}
\end{definition}

\vspace{-.5em}
The solution is equivalent to defining $\hat {\vct{\alpha}}_{\hbLa}= \hbLa^{\dag} \hat \bt_{\hbLa}$ and solving an unweighted minimum 2-norm regression with features $\X\hbLa$. This corresponds to our few-shot learning problem
\begin{align*}
   \hat {\vct{\alpha}}_{\hbLa} = \arg\min_{\vct{\alpha}} \tn{\vct{\alpha}} \quad \text{s.t.}\quad\y=\X\hbLa\vct{\alpha}
\end{align*}
from which we obtain $\hat \bt_{\hbLa} = \hbLa\hat {\vct{\alpha}}_{\hbLa}$. When there is no confusion, we can replace $\hat \bt_{\hbLa}$ with $\hat \bt$. One can easily see that $\hat \bt = \hbLa(\X\hbLa)^\dag\y$. We note that Definition~\ref{def:minl2_interpolator} is a special case of the weighted ridge regression discussed in \cite{wu2020optimal}, as stated in Observation \ref{obs:shape}. An alternative equivalence between min-norm interpolation and ridge regression can be found in \cite{muthukumar2019harmless}.

\begin{observation}\label{obs:shape}
Let $\X \in\R^{\nfs\times d}$ and $\y\in\R^{\nfs}$, define
\begin{align}
    \hat \bt_1 &= \lim_{t\rightarrow 0} \argmin_{\bt} \|\X \bt - \y\|_2^2 + t\bt^\top (\hbLa\hbLa^\top)^\dag\bt,\ \bt\in\mathrm{column~space~of~}\hbLa. \label{eq: obs 1 ridge}
\end{align}
We have that $\hat \bt_1 = \hat \bt$.
\end{observation}

\vspace{-.5em}

\vspace{-.2em}
\section{Canonical Covariance and Optimal Representation} \label{sec:opt_rep}

\vspace{-.2em}
In this section, we ask the simpler question: if the covariances $\B$ and $\bSi$ are known, what is the best choice of $\hbLa$ to minimize the risk of the interpolator from Definition~\ref{def:minl2_interpolator}? In general, the covariances are not known; however, the insights from this section help us study the more general case in Section~\ref{sec:robust}. \ys{This is equivalent to selecting the penalty of the ridge term in \eqref{eq: obs 1 ridge}, in accordance of the feature and task distributions to weight the solution $\hat\beta$.}
Define the risk as the expected error of inferring the label on the few-shot dataset,
\begin{align}
&\risk(\hbLa,\B, \bSi) = \mtx{E}_{\x,y,\bt}(y - \x^\top\hat\bt_{\hbLa})^2 
= \Eb_{\bt} (\hat \bt_{\hbLa}-\bt)^\top \bSi(\hat \bt_{\hbLa}-\bt) + \sigma^2.\label{eq:bSi_test_risk}
\end{align}
The natural choice of optimization for choosing $\hbLa$ would be to choose the weighting that minimizes the eventual risk of the learned interpolator.
\begin{equation}
\label{eq:opt_rep}
    \bLa=\arg\min_{{\hbLa}'\in\R^{d\times R}} \risk( {\hbLa}',\B, \bSi) 
\end{equation}
Since the label $y$ is bilinear in $x$ and $\beta$, we introduce whitened features $\tilde{\x}=\sbSiinv\x$ and associated task vector $\tilde{\bt}=\sbSi \bt$. This change of variables ensures $\x^T\bt=\tilde{\x}^T\tilde{\bt}$; now, the task covariance in the transformed coordinates takes the form
\[
\Bbar=\sbSi\B\sbSi,
\]
which we call the \noindent\textbf{canonical task covariance}; it captures the joint behavior of feature and task covariances $\bSi, \B$.
\an{Finding the above sentence a bit confusing i.e what does it mean to capture joint behavior?}Below, we observe that the risk in Equation~\eqref{eq:bSi_test_risk} \annew{is invariant to the change of co-ordinates that we have described above} i.e it
does not change when $\sbSi\B\sbSi$ is fixed and we vary $\bSi$ and $\B$. 
\begin{observation}[Equivalence to problem with whitened features]\label{obs:canon}
Let data be generated as in Phase 1. Denote $\Bbar=\sbSi\B\sbSi$. Then 
$\risk({\sbSiinv\hbLa},\B, \bSi) = \risk(\hbLa, \Bbar, \Ib)$.
\end{observation}
This observation can be easily verified by substituting the change-of-coordinates into  Equation~\eqref{eq:bSi_test_risk} and evaluating the risk.

The risk in \eqref{eq:bSi_test_risk} quantifies the quality of representation $\hbLa$; however it is not a manageable function of $\hbLa$ that can be straightforwardly optimized. In this subsection, we show that it is asymptotically equivalent to a different optimization problem, which can be easily solved by analyzing KKT optimality conditions. Theorem~\ref{thm:opt rep theta} characterizes this equivalence; the \textproc{computeReduction} subroutine of Algorithm~\ref{algo:opt rep} calculates key quantities that are used in specifying the reduction, and the \textproc{computeOptimalRep} subroutine of Algorithm~\ref{algo:opt rep} uses the solution of the simpler problem to obtain a solution for the original. 

\vspace{-.5em}

\begin{algorithm}[t!]
\caption{Constructing the optimal representation}
\label{algo:opt rep}
\begin{algorithmic}[1]

\Require Projection dimension $R$, noise level $\sigma$, canonical covariance $\Bbar$, task covariance $\bSi$.

\vspace{0.5em}

\Function{ComputeOptimalRep}{$R, \bSi, \Bbar, \sigma, \nfs$}
 \State $\Ub_1, \mtx{\Sigma}_{F}^R, \BbarR, \sigma_R$ = \Call{ComputeReduction}{$R, \bSi, \Bbar, \sigma$}
 \State \dred{Optimization:} Get $\btheta^*$ from \eqref{eq:theta opt}.
 \State \dred{Map to eigenvalues:} Set diagonal $\bLa_R\in\R^{R\times R}$ with entries $\bLa_{R,i} =  (1/\btheta_i^*-1) ^{-2}$. 
 \State \dred{Lifting and feature whitening:} $\bLa \leftarrow \Ub_1(\mtx{\Sigma}_{F}^R)^{-1/2}\bLa_R$.
 \State \Return $\bLa$
\EndFunction

\vspace{0.5em}

\Function{ComputeReduction}{$R, \bSi$,$\Bbar, \sigma$} \funclabel{alg:b}
 \State \dred{Get eigen-decomposition} $\Bbar = \Ub \boldsymbol{\Sigma} \Ub^\top $. 
 \State\dred{Principal eigenspace} $\Ub_1\in\R^{d\times R}$ = the first $R$ columns of $\Ub$.
 \State\dred{Top eigenvalues:} Set $\BbarR=\Ub_1^\top \Bbar\Ub_1, \bSiR=\Ub_1^\top \bSi\Ub_1$
 \State\dred{Equivalent noise level:} $\sigma_R^2 \leftarrow \sigma^2 + \tr(\Bbar)-\tr(\Bbar^R)$.
 \State \Return $\Ub_1, \mtx{\Sigma}_{F}^R, \BbarR, \sigma_R$
\EndFunction
\end{algorithmic}
\end{algorithm}

\begin{assumption}[Bounded feature covariance]\label{ass: bd}
There exist positive constants $\Sigma_{\min}$, $\Sigma_{\max}$ such that $\bSi$ is lower/upper bounded as follows: $\mathbf{0}\prec\Sigma_{\min}\Ib\preceq \bSi \preceq\Sigma_{\max}\Ib$. 
\end{assumption}
\begin{assumption}[Joint diagonalizability]\label{ass: diag}
$\bSi$ and $\B$ are diagonal matrices.\footnote{This is equivalent to the more general scenario where $\bSi$ and $\B$ are jointly diagonalizable.}
\end{assumption}


\begin{assumption}[Double asymptotic regime]\label{ass: asy}
We let the dimensions and the sample size grow as $d,R,\nfs\rightarrow\infty$ at fixed ratios $\bar{\kappa}:= {d}/{\nfs}$ and $\kappa:= {R}/{\nfs}$.
\end{assumption}


\begin{assumption}\label{ass: dist}
The joint empirical distribution of the eigenvalues of $\hbLa_R$ and $\BbarR$ is given by the average of Dirac $\delta$'s: $\frac{1}{
 R}\sum_{i=1}^R\delta_{\hbLa_{R,i},\sqrt{R}\BbarRi}$. It converges to a fixed distribution as $d\rightarrow\infty$.
\end{assumption}


With these assumptions, we can derive an analytical expression to quantify the risk of a representation $\hbLa$. We will then optimize this analytic expression to obtain a formula for the optimal representation.
\begin{theorem}[\red{Asymptotic risk equivalence}] 
\label{thm:opt rep theta}
Suppose Assumptions \ref{ass: bd}, \ref{ass: diag}, \ref{ass: asy}, \ref{ass: dist} hold.
    Let $\xi>0$ be the unique number obeying $
    \nfs = \sum_{i=1}^R\big({1+(\xi{\hbLa}^2_i)^{-1}}\big)^{-1}$.
    Define $\btheta \in \R^R$ with entries $\btheta_i=\frac{\xi\hbLa^2_i}{1+\xi\hbLa^2_i}$ and calculate $\BbarR, \sigma_R$ using the \textproc{ComputeReduction} procedure of Algorithm~\ref{algo:opt rep}. Then, define the analytic risk formula
    \vspace{-.5em}
    \begin{align}
    f(\btheta, \BbarR, \nfs)
   = \dfrac{1 }{\nfs - \|\btheta\|^2_{\red{2}}}
   \left(\nfs\sum_{i=1}^R (1 - \btheta_i)^2{\BbarRi} + (\|\btheta\|^2_{\red{2}}+1)\sigma_{R}^2\right).
    \label{optimal_lambda}
    \end{align}
    \vspace{-.5em}
    We have that
    \begin{align}\label{eq: asym thm}
        \lim_{\nfs \to \infty} f(\btheta, \BbarR, \nfs) = \lim_{\nfs \to \infty} \risk(\sbSiinv\hbLa, \B, \bSi)
    \end{align}
\end{theorem} 

The proof of Theorem~\ref{thm:opt rep theta} applies the convex Gaussian Min-max Theorem (CGMT) in \cite{thrampoulidis2015lasso} and can be found in the Appendix B.2.
We show that as dimension grows, the distribution of the estimator $\hat \beta$ converges to a Gaussian distribution and we can calculate the expectation of risk. 

Theorem~\ref{thm:opt rep theta} provides us with a closed-form risk for any linear representation. Now, one can solve for the optimal representation by computing \eqref{eq:theta opt} below. In order to do this, we propose an algorithm for the optimization problem in Appendix B.5 via a study of the KKT conditions for the problem \footnote{In Sec. \ref{sec:robust} the constraint is $\uth\le\btheta\le1-\frac{d-\nfs}{\nfs}\uth$ for robustness concerns.}. 
\begin{align}
    \btheta^* = \arg\min_{\btheta}\ f(\btheta, \B, \bSi),\ \mbox{s.t.}\ 0\le \btheta< 1, \sum_{i=1}^R\btheta_i = \nfs  \label{eq:theta opt}\tag{OPT-REP}
\end{align}

\begin{wrapfigure}{r}{0.38\textwidth}
  \centering
    \includegraphics[width=0.35\textwidth]{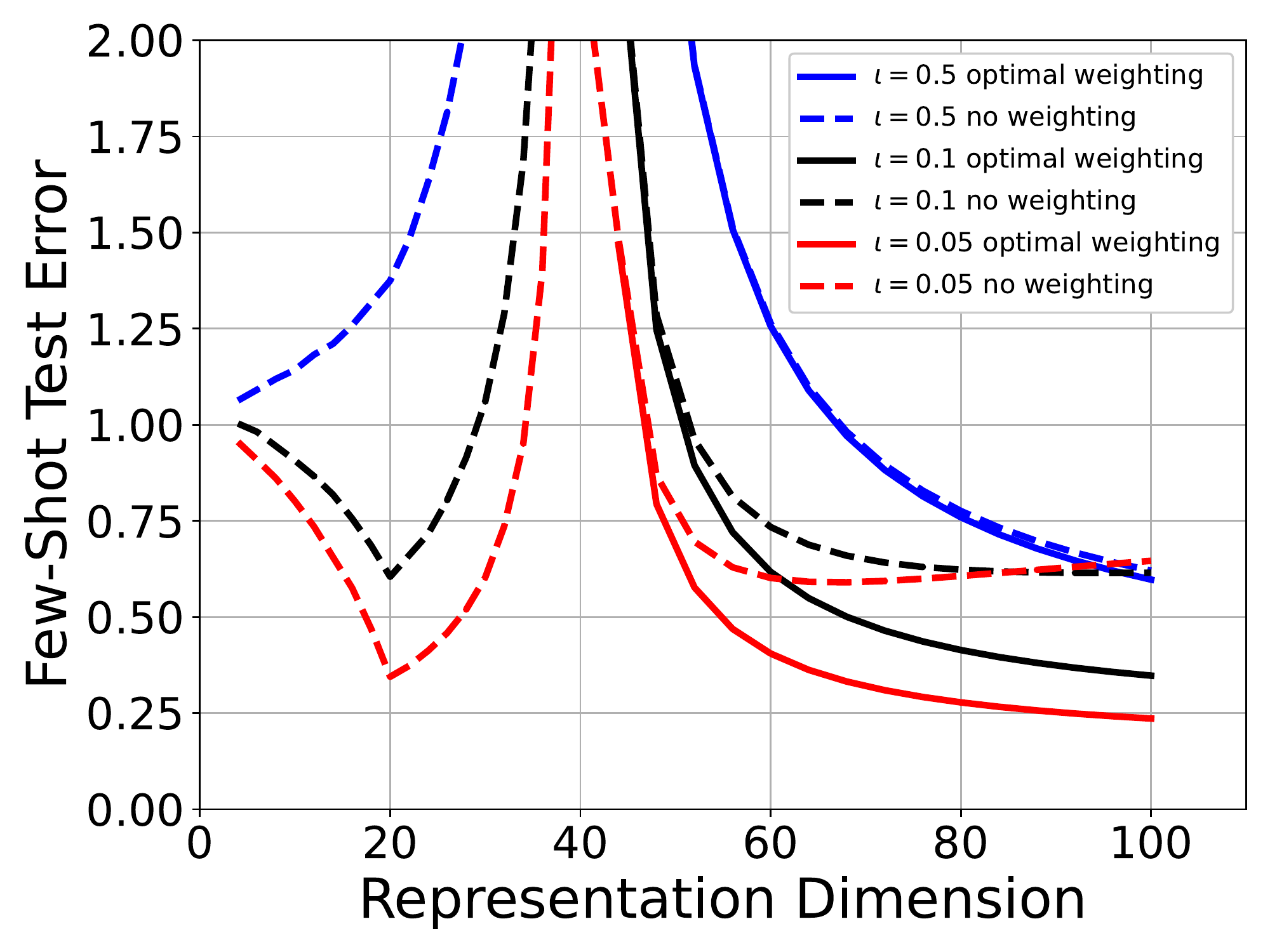}
  \caption{Theoretical risk of optimal representation. $\bSi=\Ib_{100}$, $\B=\diag{\Ib_{20},\iota\Ib_{80}}$, $\nfs=40$.}\label{fig:opt iotas}
\end{wrapfigure}

The optimal representation is\footnote{In the algorithm, $\xi=1$ and $\hbLa_{R,i} = (1/\btheta_i^*-1)^{-2}$, because $c\bLa$ for any constant $c$ gives the same $\hat \bt$. } $\bLa_{R,i} = ( (1/\btheta_i^*-1) \xi)^{-2}$. The subroutine \textproc{computeOptimalRep} in Algorithm~\ref{algo:opt rep} summarizes this procedure.

\begin{remark}\label{remark: fig 1b}
Thm.~\ref{thm:opt rep theta} states that $\risk(\sbSiinv\hbLa, \B, \bSi)$ can be arbitrarily well-approximated by $f(\btheta, \BbarR, \nfs)$ if $n_2$ is sufficiently large. In Fig. \ref{fig:fs_double_descent}(b), we set $\bSi=\Ib_{100}$, $\B=\diag{\Ib_{20},0.1\Ib_{80}}$, $\nfs=40$. The curves in Fig\ref{fig:fs_double_descent}(b) are the finite dimensional approximation of $f$ (LHS of \eqref{eq: asym thm}); the dots are empirical approximations of the risk (RHS of \eqref{eq: asym thm}). We tested two cases when $\hbLa$ is the optimal eigen-weighting or projection matrix with no weighting. Our theorem is corroborated by the observation that the dots and curves are visibly very close. The approximation is already accurate for the finite dimensional problem with just $n_2 = 40$. 
\end{remark}

\noindent\textbf{The benefit of overparameterization.} Theorem~1 leads to an optimal eigen-weighting strategy via asymptotic analysis. 
In Figure~\ref{fig:opt iotas}, we plot the effect on the risk of increasing $R$ for different shapes of task covariance; the parameter $\iota$ controls how spiked $\B$ is, with a smaller value for $\iota$ indicating increased spiked-ness. For the underparameterized problem, the weighting does not have any impact on the risk. In the overparameterized regime, the eigen-weighted learner achieves lower few-shot error than its unweighted ($\hbLa = \Ib$) counterpart, showing that
eigen-weighting becomes critical.

The eigen-weighting procedure can introduce inductive bias during few-shot learning, and helps explain how optimal representation minimizing the few-shot risk can be overparameterized with $R\gg \nfs$. We note that, an $R$ dimensional representation can be recovered by a $d$ dimensional representation matrix of rank $R$, thus the underparameterized case can never beat $d$ dimensional case in theory. The error with optimal eigen-weighting in overparameterized regime is smaller than the respective underparameterized counterpart. The error is lower with smaller $\iota$. It implies that, while $\Bbar$ gets closer to low-rank, the excess error caused by choosing small dimension $R$ (equal to the gap $\sigma_R^2-\sigma^2$ in Algo \ref{algo:opt rep}) is not as significant. 

\past{Figure \ref{fig:fs_double_descent}(b) depicts the importance of eigen-weighting: in the overparameterized regime, the blue eigen-weighted learner achieves markedly lower test-error than its unweighted ($\hbLa = \Ib$) counterpart. This eigen-weighting procedure can introduce inductive bias during few-shot learning and helps explain how optimal representation minimizing the few-shot risk can be overparameterized with $R\gg \nfs$. In contrast, for the underparameterized problem, the weighting does not have any impact.
in overparameterized regime, the smaller error corresponding to optimal eigen-weighting compared to no weighting, i.e., the eigen-weighting matrix being identity, shows that
eigen-weighting becomes critical (for underparameterized problem, weighting has no impact).}

Low dimensional representations zero out features and cause bias. By contrast, when $\Bbar\in\R^{d\times d}$ is not low rank, every feature contributes to learning with the importance of the features reflected by the weights.  
This viewpoint is in similar spirit to that of \cite{hastie2019surprises} where the authors devise a misspecified linear regression to demonstrate the benefits of overparameterization. Our algorithm allows arbitrary representation dimension $R$ and eigen-weighting. \ys{think where this should appear}\past{In the next section, we will acquire prior information by learning $\Bbar$. Thm. \ref{thm:opt rep theta} and \eqref{eq:theta opt} infuses this information to design the best eigen-weighting.}

\section{Representation Learning} 
\label{sec:meta train}


In this section, we will show how to estimate the useful distribution in representation learning phase that enables us to calculate eigen-weighting matrix $\bLa$. Note that $\bLa$ depends on the canonical covariance $\Bbar=\sbSi\B\sbSi$.  Learning the $R$-dimensional principal subspace of $\Bbar$ enables us\footnote{We also need to estimate $\bSi$ for whitening. Estimating $\bSi$ is rather easy and incurs smaller error compared to $\Bbar$. The analysis is provided in the first part of Appendix B.} to calculate $\bLa$. Denote this subspace by $\red{\Sbar}$.

\textbf{Subspace estimation vs. inductive bias.} The subspace-based representation $\Sbar$ has degrees of freedom$=Rd$. When $\Bbar$ is exactly rank $R$ and features are whitened, \cite{tripuraneni2020provable} provides a sample-complexity lower bound of $\Omega(Rd)$ examples and gives an algorithm achieving $\order{R^2d}$ samples. However, in practice, deep nets learn good representations despite overparameterization. In this section, recalling our \textbf{Q2}, we argue that the inductive bias of the feature distribution can implicitly accelerate learning the canonical covariance. \red{This differentiates our results from most prior works such as \cite{kong2020meta,kong2020robust,tripuraneni2020provable} in two aspects:}

\begin{enumerate}[leftmargin=*]
    \item Rather than focusing on a \emph{low dimensional} subspace and assuming $N\gtrsim Rd$, we can estimate $\Bbar$ or $\Sbar$ in the overparameterized regime $N\lesssim Rd$.
    \item Rather than assuming whitened features $\bSi = \Ib$ and achieving a sample complexity of $R^2d$, our learning guarantee holds for arbitrary covariance matrices $\bSi,\B$. The sample complexity depends on \emph{effective rank} and can be \red{arbitrarily} smaller than DoF. We showcase our bounds via a spiked covariance setting in Example \ref{example:spike} below. 
\end{enumerate}

For learning $\Bbar$ or its subspace $\red{\Sbar}$, we investigate the method-of-moments (MoM) estimator.
\vspace{-1em}
\begin{definition}[MoM Estimator]
\label{def:method_moment_estimator}
For $1\leq i\leq T$, define $\hat \bb_{i,1} = 2\nspt^{-1}\sum_{j=1}^{\nspt/2}\yij\xij$, $\hat \bb_{i,2} = 2\nspt^{-1}\sum_{j=\nspt/2+1}^{\nspt}\yij\xij$. Set
\vspace{-1em}
\begin{align*}
    \hat \bM &= \nspt^{-1}\sum_{i=1}^{\ntask} (\bb_{i,1}\bb_{i,2}^\top + \bb_{i,2}\bb_{i,1}^\top),
\end{align*}
\vspace{-1em}

The expectation of $\hat \bM$ is equal to $\bM = \bSi\B\bSi$.
\end{definition}

\textbf{Inductive bias in representation learning:} Recall that canonical covariance $\Bbar=\sbSi\B\sbSi$ is the attribute of interest. However, feature covariance $\sbSi$ term implicitly modulates the estimation procedure because the population MoM is not $\Bbar$ but $\bM = \sbSi\Bbar\sbSi$. For instance, when estimating the principle canonical subspace $\Sbar$, the degree of alignment between $\bSi$ and $\Bbar$ can make or break the estimation procedure: If $\bSi$ and $\Bbar$ have \emph{well-aligned} principal subspaces, $\Sbar$ will be easier to estimate since $\bSi$ will amplify the $\Sbar$ direction within $\M$.

We verify the inductive bias on practical image dataset, reported in Appendix A. We assessed correlation coefficient between covariances $\Bbar,\bSi$ via the canonical-feature alignment score defined as the correlation coefficient 
\[
\rho(\bSi,\Bbar):=\frac{\big<\bSi,\Bbar\big>}{\|\bSi\|_F\|\Bbar\|_F}=\frac{\text{trace}(\bM)}{\|\bSi\|_F\|\Bbar\|_F}.
\]
\red{Observe that, the MoM estimator $\bM$ naturally shows up in the alignment definition because the inner product of $\Bbar,\bSi$ is equal to $\text{trace}(\bM)$. This further supports our inductive bias intuition.}
As reference, we compared it to canonical-identity alignment defined as  $\frac{\text{trace}(\Bbar)}{\sqrt{d}\|\Bbar\|_F}$ (replacing $\bSi$ with $\Ib$). The canonical-feature alignment score is higher than the canonical-identity alignment score. This significant score difference exemplifies how $\bSi$ and $\Bbar$ can synergistically align with each other (inductive bias).
This alignment helps our MoM estimator defined below, illustrated by Example \ref{example:spike} (spiked covariance).






 In the following subsections, let $\ntot = \nspt \ntask$ refer to the total tasks in representation-learning phase. 
 Let $\TX = \tr(\bSi)$, $\Tbt = \tr(\B)$, and $\TbtX = \tr(\Bbar)$. 
Define the approximate low-rankness measure of feature covariance by\footnote{The $(\rf+1)$-th eigenvalue is smaller than $\rf/d$. Note the top eigenvalue is $1$.}
\begin{align*}
    \rf = \min \ \rf',\ \mbox{s.t.~} \rf'\in\{1,...,d\},\  \rf'/d \ge \lambda_{\rf'+1}(\bSi)
\end{align*}
 
 We have two results for this estimator.
  \begin{enumerate}[leftmargin=*]
     \item Generally, we can estimate $\bM $ with $\Order(\TX\TbtX^2)$ samples.
     \item 
     Let $\nspt\ge \rt$, we can estimate $\bM$ with $\Order(\rf\TbtX)$ samples.
 \end{enumerate}
  Paper \cite{tripuraneni2020provable} has sample complexity $\Order(dr^2)$ ($r$ is exact rank). Our sample complexity is $\Order(\TX\TbtX^2)$. $\TX , \TbtX$ can be seen as effective ranks and our bounds are always smaller than \past{$d,r$ of the setting in} \cite{tripuraneni2020provable}. We will discuss later in Example \ref{example:spike}. Our second result says when $\nspt\ge \rt$, our sample complexity achieves the $\order{dr}$ which is proven a lower bound in \cite{tripuraneni2020provable}. 
  \setlength{\textfloatsep}{.3em}
\begin{table}[t!]
\centering
\scalebox{0.85}{
\begin{tabular}{|c|c|c|c|c|c|c|} 
 \hline
  feature cov & \multicolumn{3}{|c|}{$\bSi = \Ib$, $\B=\diag{\Ib_{\rt},{\bf 0}}$} & \multicolumn{3}{|c|}{
  \makecell{$\bSi = \diag{\Ib_{\rf},\iotaF\Ib_{d-\rf}}$,\\ $\B=\diag{\Ib_{\rt},\iotaT\Ib_{d-\rt}}$}
  } \\
   \hline
 estimator & sample $\ntot$ & sample $\nspt$ & error & sample $\ntot$ & sample $\nspt$ & error\\
 \hline\hline
 MoM & $d\rt^2$ &$1$ & $(d\rt^2/\ntot)^{1/2}$ & 
 $\barrf \barrt^2$ & $1$ & $(\barrf \barrt^2/\ntot)^{1/2}$\\
 \hline
 MoM & $d\rt$ &  $\rt$ & $(\rt/\nspt)^{1/2}$ & $\barrf\barrt$ & $\barrt$ & $(\barrt/\nspt)^{1/2}$ \\
 \hline
\end{tabular}
}
\vspace{.5em}
\caption{\textbf{Right side:} Sample complexity and error of MoM estimators. $\rf$ ($\rt$) is the dimension of the principal eigenspace of the feature (task) covariance. $\barrf = \rf + \iotaF(d-\rf)$, $\barrt = \rt + \iotaT(d-\rt)$ are the effective ranks. \textbf{Left side:} This is the well-studied setting of identity feature covariance and low-rank task covariance. Our bound in the second row is the first result to achieve optimal sample complexity of ${\cal{O}}(ds_{T})$ (cf.~\cite{tripuraneni2020provable, kong2020meta}).
}
\label{table:1}
\end{table}

\begin{theorem}\label{thm:mom 1}
Let data be generated as in Phase 1. Assume $\|\bSi\|, \|\B\|\ = 1$ for normalization\footnote{This is simply equivalent to scaling $\yij$, which does not affect the normalized error $\|\hat\bM - \bM\| / \|\bM\|$. In the appendix we define $\Snorm = \max\{\|\bSi\|, \|\B\|\}$ and prove the theorem for general $\Snorm$.}.
\begin{enumerate}[leftmargin=*]
    \item Let $\nspt$ be a even number. Then with probability \red{at least $1 - \ntot^{-100}$,}
\begin{align*}
   \| \hat \bM - \bM \| \lesssim (\TbtX + \sigma^2)\sqrt{\frac{\TX}{\ntot}} + \sqrt{\frac{\Tbt}{\ntask}}.
\end{align*}
\item 
Assume $T\ge \rf$. If $\nspt \gtrsim \TbtX+\sigma^2$,
then with probability at least $1 - CT^{-100}$
\begin{align*}
\| \hat \bM - \bM \| \lesssim \left((\TbtX+\sigma^2)/\nspt\right)^{1/2}. 
\end{align*}
\end{enumerate}
\red{Denote the top-$R$ principal subspaces of $\bM, \hat \bM$ by $\bMt, \hatbMt$ and assume the eigen-gap condition $\lambda_R(\bM) - \lambda_{R+1}(\bM) > 2\| \hat \bM - \bM \|$. Then a direct application of Davis-Kahan Theorem \cite{davis1970rotation} bounds the subspace angle as follows \[\text{angle}( \bMt, \hatbMt) \lesssim \| \hat \bM - \bM \|/(\lambda_R(\bM) - \lambda_{R+1}(\bM)).\]}
\end{theorem}

\begin{wrapfigure}{r}{0.38\textwidth}
\vspace{-1em}
  \centering
    \includegraphics[width=0.35\textwidth]{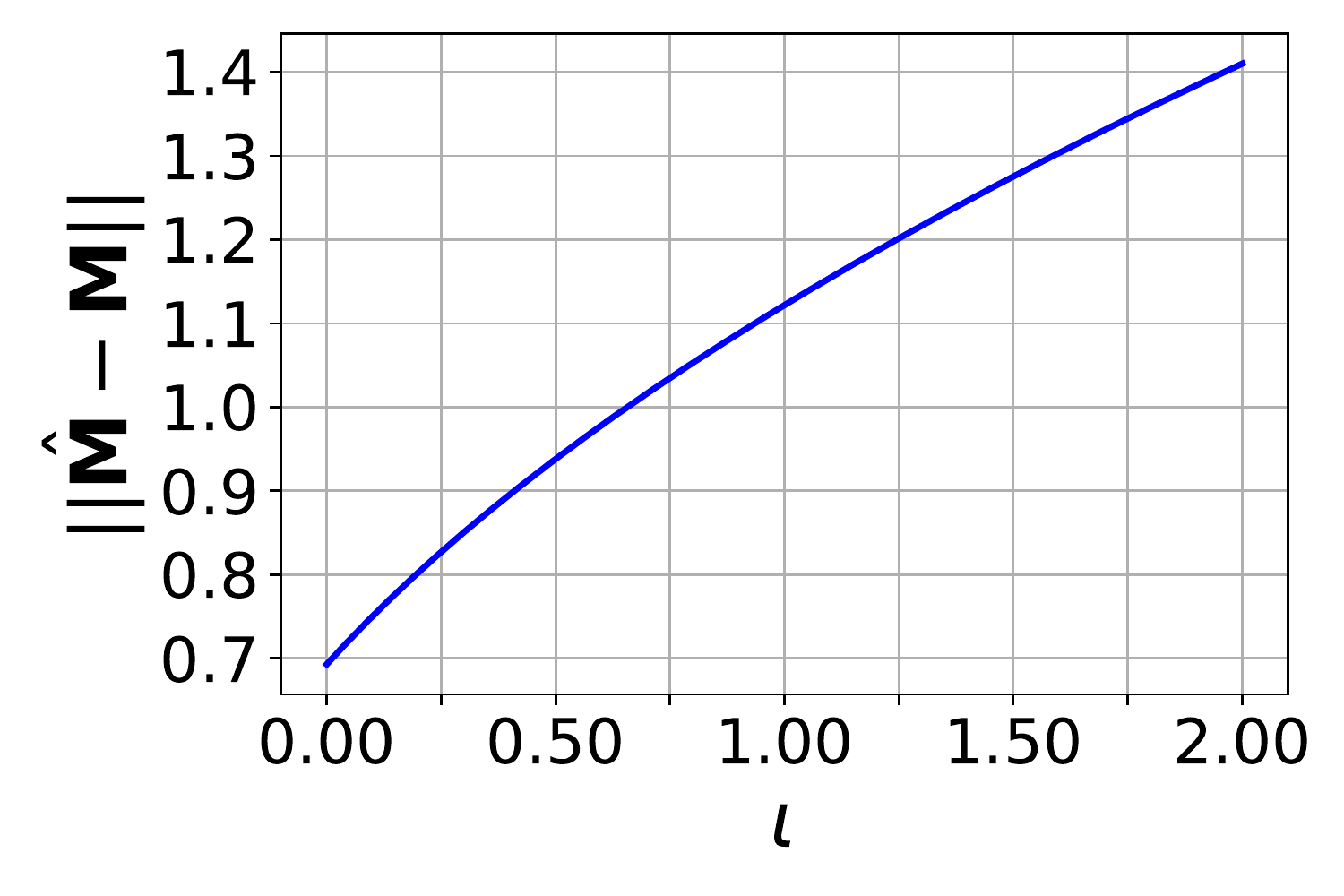}
    \vspace{-10pt}
  \caption{Error of MoM estimator \so{How is there still an estimation error with $\iota=0$?} }\label{errM}\vspace{-10pt}
\end{wrapfigure}
\noindent\emph{Estimating eigenspace of canonical covariance.} Note that if $\bSi$ and $\B$ are aligned, (e.g.~Example \ref{example:spike} below \red{with $\rf=\rt=R$}), then $\bMt=\red{\Sbar}$ is exactly the principal subspace of $\Bbar$. Theorem \ref{thm:mom 1} indeed gives estimation error for the principal subspace of $\Bbar$. Note that, such alignment is a more general requirement compared to related works which require whitened features \cite{tripuraneni2020provable, kong2020meta}.

\begin{example}[Spiked $\Bbar$, Aligned principal subspaces]\label{example:spike}Suppose the spectra of $\bSi$ and $\Bbar$ \red{are bimodal as follows} $\bSi = \diag{\Ib_{\rf},\iotaF\Ib_{d-\rf}}$, $\B=\diag{\Ib_{\rt},\iotaT\Ib_{d-\rt}}$. Set statistical error $\text{Err}_{T,N}:=\sqrt{{\barrt^2\barrf}/{\ntot}} + \sqrt{{\barrt}/{\ntask}}$.
When $\iotaT, \iotaF<1$, $\rf\geq \rt$, the recovery error of \red{$\Bbar$ and its principal subspace $\Sbar$} are bounded as
\[
     \text{angle}(\hatbMt, \Sbar)
     \lesssim \text{Err}_{T,N}+ \iotaF^2\iotaT\quad  \mbox{and}\quad \|\hat \bM -  \Bbar\| \lesssim  \text{Err}_{T,N} + \iotaF\iotaT.
\]
\end{example}
The estimation errors for $\Bbar,\Sbar$ are controlled in terms of the effective ranks and the spectrum tails $\iota_F,\iota_T$. Typically $\rf\rt\gtrsim \nspt$ so $\sqrt{{\barrt^2\barrf}/{\ntot}}$ term dominates the statistical error in practice. In Fig. \ref{errM} we plot the error of estimating $\bM$ (whose principal subspace coincides with $\Bbar$). $\bSi=\diag{\Ib_{30},\iota \Ib_{70}}$, $\B=\diag{\Ib_{30},{\bf 0}_{70}}$. $\ntask = \ntot = 100$. We can see that the error increase with $\iota$ \so{Please check why $\iota=0$ has error if you are choosing $R=30$}.

\section{Robustness of Optimal Representation and Overall Meta-Learning Bound}\label{sec:robust}
\vspace{-.2em}
In Section~\ref{sec:opt_rep}, we described the algorithm for computing the optimal representation with \emph{known} distributions of features and tasks. In Section~\ref{sec:meta train}, we proposed the MoM estimator in representation learning phase to estimate the unknown covariance matrices. In this section, we study the algorithm's behaviors when we calculate $\hbLa$ using the \emph{estimated} canonical covariance, rather than the full-information setting of Section~\ref{sec:opt_rep}.

Armed with the provably reliable estimators of Section~\ref{sec:meta train}, we can replace $\Bbar$ and $\bSi$ in Algorithm~\ref{algo:opt rep} with our estimators. In this section, we inquire: how does the estimation error in covariance-estimation in representation learning stage affect the downstream few-shot learning risk? That says, we are interested in\footnote{Note that Sec.6 of \cite{wu2020optimal} gives the exact value of $\risk(\bLa,\B,\bSi)$ so we have an end to end error guarantee.}  $\risk(\hbLa,\B,\bSi) - \risk(\bLa,\B,\bSi)$.

Let us replace the constraint in \eqref{eq:theta opt} by $\uth\le\btheta\le1-\frac{d-\nfs}{\nfs}\uth$. This changes the ``optimization" step in Algorithm~\ref{algo:opt rep}. 
Theorem~\ref{thm:e2e} does not require an explicit computation of the optimal representation by enforcing $\uth$. Instead, we use the robustness of such a representation (due to its well-conditioned nature) to deduce its stability. That said, for practical computation of optimal representation, we simply use Algorithm~\ref{algo:opt rep}. We can then evaluate $\uth$ after-the-fact as the minimum singular value of this representation to apply Theorem~\ref{thm:e2e} without assuming an explicit $\uth$. 

Let $\hbLa_{\uth}(R) = \textproc{ComputeOptimalRep}(R, \bSi, \hat\bM, \sigma, \nfs)$ denote the estimated optimal representation and $\bLa_{\uth}(R) = \textproc{ComputeOptimalRep}(R, \bSi, \Bbar, \sigma, \nfs)$ denote the true optimal representation, which cannot be accessed in practice. Below we present the bound of the whole meta-learning algorithm. It shows that a bounded error in representation learning leads to a bounded increase on the downstream few-shot learning risk, thus quantifying the robustness of few-shot learning to errors in covariance estimates.

\begin{wrapfigure}{r}{0.4\textwidth}
\centering
\includegraphics[width = 0.38 \textwidth]{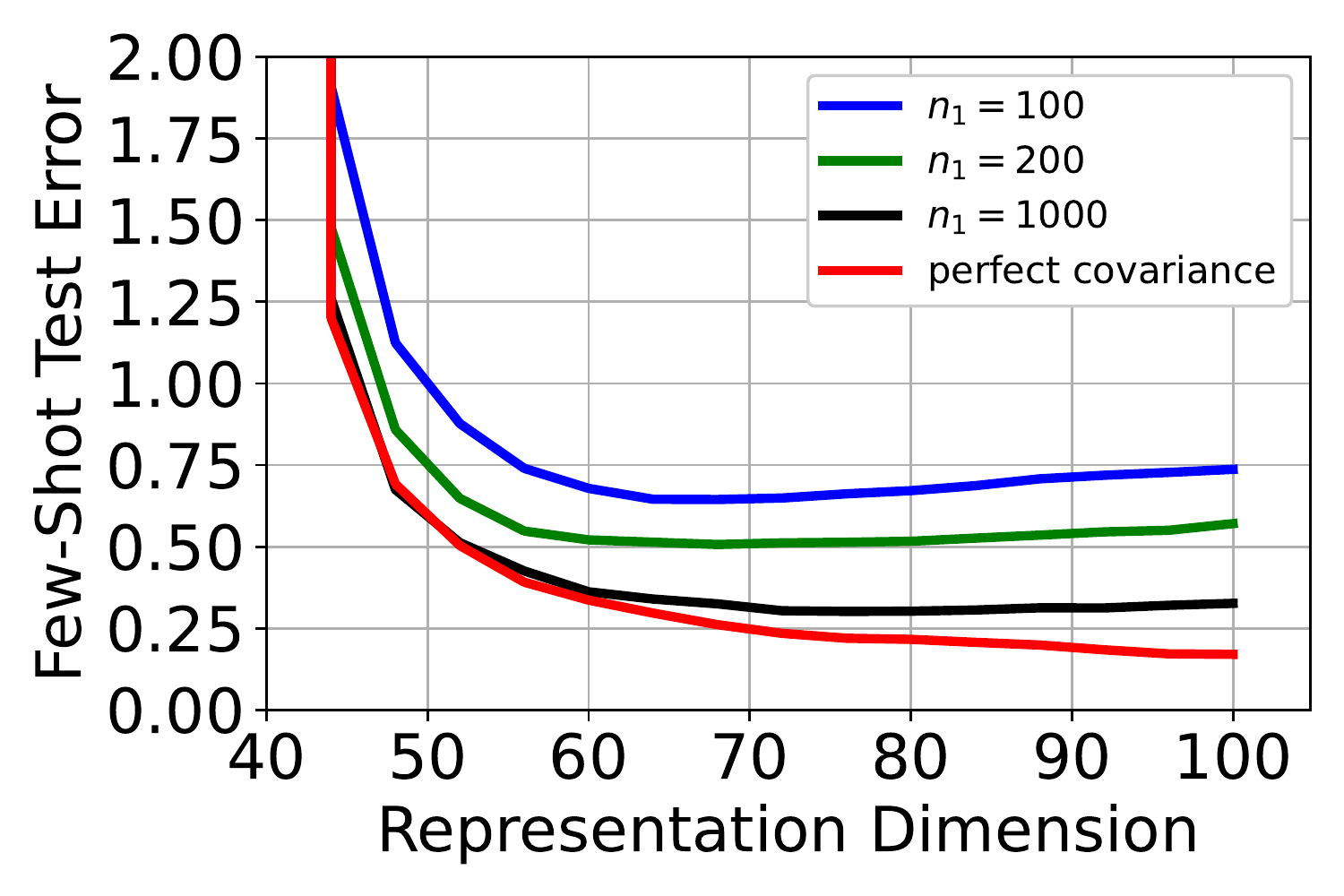}
\vspace{-.3em}
\caption{End to end learning guarantees. $d=100, \nfs=40, \ntask=200$, $\B=(\Iden_{20},0.05\cdot\Iden_{80})$, $\bSi=\Iden_{100}$.}  
\label{fig2}
\end{wrapfigure}

\begin{theorem}\label{thm:e2e}
 Let $\hbLa_{\uth}(R)$, $\bLa_{\uth}(R)$ be as defined above, and $\TX = \tr(\bSi)$, $\Tbt = \tr(\B), \TbtX = \tr(\Bbar)$. The risk of meta-learning algorithm satisfies\footnote{The bracketed expression applies first conclusion of Theorem \ref{thm:e2e}. One can plug in the second as well.} 
\vspace{-0.2em}
\begin{align*}
     \risk(\hbLa_{\uth}(R), \B, \bSi) - \risk(\bLa_{\uth}(R),\B, \bSi) \lesssim 
    \frac{\nfs^2}{d(R-\nfs)(2\nfs - R\uth)\uth} \left[(\TbtX + \sigma^2)\sqrt{\frac{\TX}{\ntot}} + \sqrt{\frac{\Tbt}{\ntask}}\right].
\end{align*}
\end{theorem}
\vspace{-.5em}

Notice that as the number of previous tasks $T$ and total representation-learning samples $N$ observed increases, the risk of the estimated $\hbLa_{\uth}(R)$ approaches that of the optimal $\bLa_{\uth}(R)$ as we expect.  The result only applies to the overparameterized regime of interest $R > \nfs$. The expression of risk in the underparameterized case
is different, and covered by the second case of Equation(4.4) in \cite{wu2020optimal}. We plot it in Fig~\ref{fig:fs_double_descent}(b) on the left side of the peak as a comparison.

\textbf{Risk with respect to PCA level $R$.} In Fig. \ref{fig2}, we plot the error of the whole meta-learning algorithm. We simulate representation learning and get $\hat \bM$, use it to compute $\hbLa$ and plot the theoretical downstream risk (experiments match, see Fig. \ref{fig:fs_double_descent} (b)). Mainly, we compare the behavior of Theorem~ \ref{thm:e2e} with different $R$. When $R$ grows, we search $\hbLa$ in a larger space. The optimal $\hbLa$ in a feasible \emph{sub}set is always no better than searching in a larger space, thus the risk decreases with $R$ increasing. At the same time, representation learning error increases with $R$ since we need to fit a matrix in a larger space. In essence, this result provides a theoretical justification on a sweet-spot for the optimal representation. $d=R$ is optimal when $\ntot=\infty$, i.e., representation learning error is $0$. As $\ntot$ decreases, there is a tradeoff between learning error and truncating small eigenvalues. Thus choosing $R$ adaptively with $\ntot$ can strike the right bias-variance tradeoff between the excess risk (variance) and the risk due to suboptimal representation.

\vspace{-.5em}
\section{Conclusion}\label{sec:conclusion}
\vspace{-.2em}
In this paper, we study the sample efficiency of meta-learning with linear representations. We show that the optimal representation is typically overparameterized and  outperforms subspace-based representations for general data distributions. We refine the sample complexity analysis for learning arbitrary distributions and show the importance of inductive bias of feature and task. Finally we 
provide an end-to-end bound for the meta-learning algorithm showing the tradeoff of choosing larger representation dimension v.s. robustness against representation learning error.

\bibliography{refs.bib}
\bibliographystyle{plain}

\appendix
\clearpage
\section{Numerical verification of inductive bias for representation learning}

\begin{figure}[htbp!]
\centering
\subfigure{\includegraphics[width =0.45\textwidth]{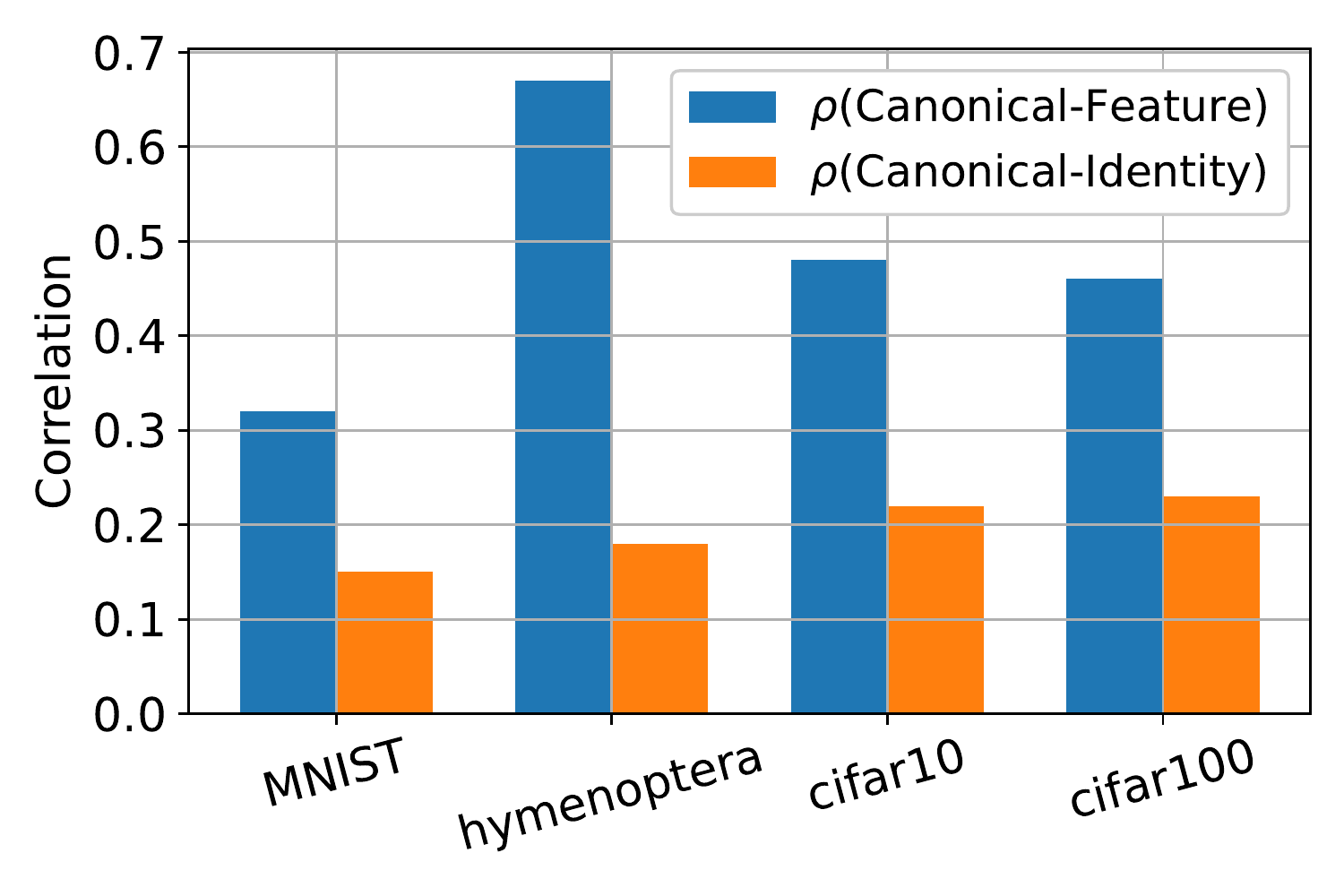}
}
\subfigure{\includegraphics[width =0.45\textwidth]{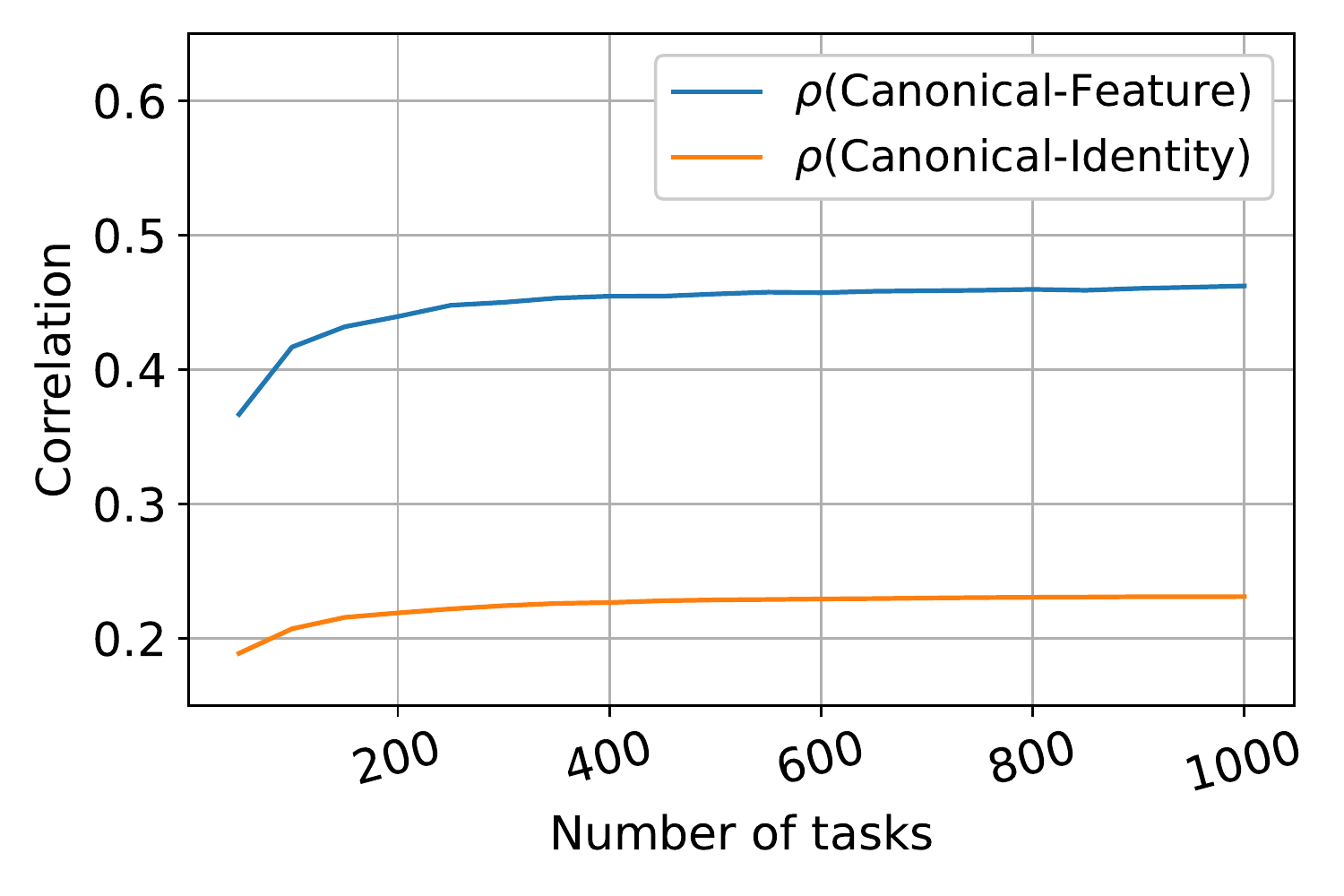}
}
\caption{(a) Alignment of feature-task on image classification models. For MNIST, we train 45 linear pairwise classifiers between each two classes. We apply the pretrained ResNet classification model on the other three datasets, compute the (last layer) feature/task covariances and get the alignments. The alignment is a measure of correlation which is denoted by $\rho$ here. (b) We use the cifar100 dataset, take the pretrained ResNet18 network and vary the number of tasks (i.e., varying the number of output classes of the neural net, also equivalent to number of rows of the last layer matrix $B$ defined below). The alignments increase with number of tasks.} \label{fig:exp data}
\end{figure}

We add a figure with experiments on a few image datasets. We take the pretrained ResNet18 neural network, and feed the images into it. For every image, we take the last (closest to output) layer output as the feature $\x$, which is of dimension $d=512$. The weights of the last layer are the tasks, which is a $\ntask\times d$ matrix (We call it $B$). $\ntask=1000$, each row of $B$ is a task vector. Then $Bx\in\R^{\ntask}$ generates the label, whose each entry corresponds to each class. We calculate the feature and task covariance, as well as the alignments defined in Sec. \ref{sec:meta train}. We can clearly see the inductive bias of every dataset.

\section{Analysis of optimal representation}\label{s:few shot app}
\subsection{Proof of Observation \ref{obs:shape} and equivalent noise}\label{s:proof obs app}
\setcounter{observation}{0}
\begin{observation}\label{obs:app}
Let $\hbLa\in\R^{d\times R}$, $\X \in\R^{\nfs\times d}$ and $\y\in\R^\nfs$, and define
\begin{align}
    \hat \bt &= \hbLa(\X \hbLa)^\dag \y,\\
    \hat \bt_1 &= \lim_{t\rightarrow 0} \argmin_{\bt} \|\X\bt - \y\|^2 + t\bt^\top (\hbLa\hbLa^\top)^\dag\bt
\end{align}
Then $\hat \bt_1 = \hat \bt$.
\end{observation}

\begin{proof}
Denote the SVD $(\X\hbLa)^\top  = \Ub\bSigma \Vb^\top$, where
$\Ub\in\R^{R\times R},\bSigma \in\R^{R\times \nfs}, \Vb\in\R^{\nfs\times \nfs}$. 

\begin{align*}
     \hat \bt_1 &= \lim_{t\rightarrow 0} \argmin_{\bt} \|\X \bt - \y\|^2 + t\bt^\top (\hbLa\hbLa^\top)^\dag\bt\\
     &= \lim_{t\rightarrow 0} (\X^\top\X + t(\hbLa\hbLa^\top)^\dag)^{-1} \X \y\\
     &= \lim_{s\rightarrow \infty} s \hbLa (s \hbLa^\top \X^\top\X \hbLa + I)^{-1} \hbLa^\top \X^\top \y\\
     &= \lim_{s\rightarrow \infty} s \hbLa (s \Ub\bSigma \Vb^\top \Vb\bSigma^\top \Ub^\top + I)^{-1} \Ub\bSigma \Vb^\top \y\\
     &= \lim_{s\rightarrow \infty} s \hbLa (s \Ub\diag{\bSigma^\top \bSigma + I_{\nfs}, I_{R-\nfs}}\Ub^\top)^{-1} \Ub\bSigma \Vb^\top \y\\
     &= \lim_{s\rightarrow \infty} \hbLa \Ub (\diag{\bSigma^\top \bSigma, I_{R-\nfs}/s})^{-1} \bSigma \Vb^\top \y.\\
     &= \hbLa(\X \hbLa)^\dag \y
\end{align*}
\end{proof}

The risk of $\hat \bt$ is given by
\begin{align*}
\text{risk}(\hat \bt) = \mtx{E}(y - \x^\top\hat\bt) &= \Eb (\hat \bt-\bt)^\top \bSi(\hat \bt-\bt) + \sigma^2.
\end{align*}

In Sec. \ref{s:asym}, we study the asymptotic optimal representation. Below, we characterize the properties of the problem for fixed $\bt$ and arbitrary input covariance $\bSi$. We first go over this and then discuss how to obtain the optimal representation $\bLa$ minimizing test risk.

\begin{remark} \label{s:proj app}
\textbf{Projection onto $R$ dimensional subspace. }
For the remaining proof after this part, we will mainly analyze the relation between $\hbLa_R$ and $\btheta$ in Thm. \ref{thm:opt rep theta}, which lie in an $R$ dimensional subspace. Here we will build the connection from the $d$ dimensional problem to $R$ dimensional, mainly computing the equivalent noise below. The equivalent noise consists of original noise and the extra noise caused by PCA truncation.

Let $\x_{R}$ be the projection of $\x$ onto the $R$-dimensional subspace spanned by columns of $\mtx{U}_1$, and $\x_{R^\perp}$ is the projection of $\x$ onto the orthogonal complement. Namely, $\x_{R}=\mtx{U}_1^\top\x \in \R^{R}$ and $\x_{R^\perp}=\mtx{U}_2^\top\x\in \R^{(d-R)}$. Similarly we can define $\bt_{R}$ and $\bt_{R^\perp}$. Thus,
\begin{align}
y = \x^\top\bt + \eps = \x_{R}^\top\bt_{R} + \x_{R^\perp}^\top\bt_{R^\perp} + \eps
\end{align}
We can treat $\eps_{R}=\x_{R^\perp}^\top\bt_{R^\perp} + \eps$ as the new noise, and try to solve for $\bt_{R}$. Then define $\mathbf{\Sigma}_{T,R^\perp}$ as the matrix containing the same eigenvectors as $\B$ and the top $R$ eigenvalues are zeroed out, our noise variance becomes $\sigma_R^2=\sigma ^2 + \Eb(\|\x_{R^\perp}\|^2\|\bt_{R^\perp}\|^2)= \sigma^2 + \tr(\Bbar)-\tr(\Bbar^R)$ in our algorithm. If we are still in overparameterized regime, namely $R>\nfs$, then we define optimal representation on top of it. 

In summary, the $R$-SVD truncation reduces the search space of $\hbLa$ into $R$ dimensional space, where the covariance of the noise in $\y$ increases from $\sigma^2 \Ib$ to $\sigma^2_R\Ib$. 
\end{remark}

\subsection{Distributional characterization of least norm solution}\label{s:asym}
In this part, for simplicity of discussion, we focus on the $R$ dimensional space while omitting the projection step, and the equivalence of a diagonal eigen-weighting matrix $\hbLa_R\in\R^{R\times R}$ and $\btheta\in\R^R$ in Thm. \ref{thm:opt rep theta}. Here, we assume a truncated feature matrix $\Xasym\in\R^{n \times R}$ where the feature is projected into an $R$ dimensional space. 

Define $\Xasym\in \R^{n \times R}, \yasym\in\R^n$. We study the following least norm solution of the least squares problem
\begin{align}
    \hat \bt = \arg\min_{\bt'}\ \|\bt'\|,\quad \mbox{s.t.,}~\Xasym\bt' = \yasym 
\end{align}
\begin{assumption}\label{ass:linear} Assume the rows of $\tilde \X$ are independently drawn from $\Nn(0, \bSiasym)$.
We focus on a double asymptotic regime where $R,n\rightarrow\infty$ at fixed overparameterization ratio $\kappa:=R/n>0$.  
\end{assumption}
\begin{restatable}{assumption}{asstwo}\label{ass:inv} {The covariance matrix $\bSiasym$ is diagonal} and there exist constants $\Sigma_{\min},\Sigma_{\max}\in(0,\infty)$ such that:
$
0\prec\Sigma_{\min}\Ib \preceq \bSiasym \preceq \Sigma_{\max}\Ib.
$
\end{restatable}

\begin{restatable}{assumption}{assthree}\label{ass:mu}
The joint empirical distribution of $\{(\lambda_i(\bSiasym),\bt_i)\}_{i\in[R]}$ converges in {Wasserstein-k} distance to a probability distribution $\mu$ on $\R_{>0}\times\R$ for some $\ntask\geq 4$. That is
$
\frac{1}{R}\sum_{i\in[R]}\delta_{(\lambda_i(\bSiasym),\bt_i)} \stackrel{W_k}{\Longrightarrow} \mu.
$
\end{restatable}
\begin{definition}[Asymptotic distribution characterization -- Overparameterized regime]\label{def:Xi}
\cite{thrampoulidis2015lasso} Let random variables $(\Sigma,B)\sim \mu$ (where $\mu$ is defined in Assumption \ref{ass:mu}) and fix $\kappa>1$. Define parameter $\xi$ as the unique positive solution to the following equation
\begin{align}\label{eq:ksi}
\E_{\mu}\Big[ \big({1+(\xi\cdot\Sigma)^{-1}}\big)^{-1} \Big] = {\kappa^{-1}}\,.
\end{align}
Define the positive parameter $\gamma$ as follows:
\begin{align}
\label{eq:gamma}
\hspace{-0.1in}\gamma := 
\Big({\sigma^2 + \E_{\mu}\Big[\frac{B^2\Sigma}{(1+\xi\Sigma)^2}\Big]}\Big)\Big/\Big({1-\kappa\E_{\mu}\Big[\frac{1}{\left(1+(\xi\Sigma)^{-1}\right)^2}\Big]}\Big).
\end{align}
With these and $H\sim\Nn(0,1)$, define the random variable
\begin{align}
X_{\kappa,\sigma^2}(\Sigma,B,H) := \Big(1-\frac{1}{1+ \xi\Sigma}\Big) B + \sqrt{\kappa}\frac{\sqrt{\gamma}\,\Sigma^{-1/2}}{1+(\xi\Sigma)^{-1}} H, \label{eq:X}
\end{align}
and let $\Pi_{\kappa,\sigma^2}$ be its distribution.
\end{definition}

\begin{restatable}[Asymptotic distribution characterization -- Overparameterized linear Gaussian problem]{theorem}{mainthm}\label{thm:master_W2} 
\cite{thrampoulidis2015lasso} Fix $\kappa>1$ and suppose Assumptions \ref{ass:inv} and \ref{ass:mu} hold. Let 
\[
\frac{1}{R}\sum_{i=1}^{R}\delta_{\sqrt{R}{\hat \bt}_i,\sqrt{R}\bt_i,\bSiasymi}
\] be the joint empirical distribution of $({\sqrt{R}\hat \bt},\sqrt{R}\bt,\bSiasym)$ and it converges to a fixed distribution as dimension grows. Let $f:\R^3\rightarrow\R$ be a function in $\rm{PL}(2)$. We have that
\begin{align}\label{eq:thm}
\hspace{-0.1in}\frac{1}{R} \sum_{i=1}^{R} f(\sqrt{R}\hat\bt_i,\sqrt{R}\bt_i,\bSiasymi) \rP \E\left[f(X_{\kappa,\sigma^2},B,\Sigma) \right].
\end{align}
In particular, the risk is given by
\begin{align}
\text{risk}({\hat \bt}_n)&\rP \E[\Sigma (B-X_{\kappa,\sigma^2})]+\sigmaze^2\\
&=\E[\frac{\Sigma}{(1+\xi\Sigma)^2} B^2+\frac{\kappa \gamma}{(1+(\xi\Sigma)^{-1})^2}] + \sigmaze^2.
\end{align}
\end{restatable}

\subsection{Finding Optimal Representation}\label{s:opt rep app}
Now, for simplicity (and actually without losing generality) assume $\bSiasym=\Iden$. This means that empirical measure of $\bSi$ trivially converges to $\Sigma=1$. With the representation $\bLa$ with asymptotic distribution $\Lambda$, the ML problem has the following mapping
\[
\bt\rightarrow \hbLa_R^{-1}\bt\quad\text{and}\quad \bSiasym\rightarrow \hbLa_R\bSiasym\hbLa_R.
\]
This means the empirical measure converges to the following mapped distributions
\[
B\rightarrow \bar{B}=\Lambda^{-1}B\quad \text{and}\quad \Sigma=1\rightarrow \bar{\Sigma}=\Lambda^2\Sigma=\Lambda^2.
\]
\noindent \textbf{Our question:} Craft the optimal distribution $\Lambda$ to minimize the representation learning risk. Specifically, for a given $(B,\Lambda)$ pair, we know from the theorem above that
\begin{align}
\text{risk}^{\hbLa_R}({\hat \bt}_n) &\rP\E[\frac{\bar{\Sigma}}{(1+\xi\bar{\Sigma})^2} \bar{B}^2+\frac{\kappa \gamma}{(1+(\xi\bar{\Sigma})^{-1})^2}]+ \sigmaze^2\\
&=\E[\frac{B^2}{(1+\xi\Lambda^2)^2}+\frac{\kappa \gamma}{(1+(\xi\Lambda^2)^{-1})^2}]+ \sigmaze^2.
\end{align}
Thus, the optimal weighting strategy (asymptotically) is given by the distribution
\[
\Lambda^*=\arg\min_{\Lambda}\E[\frac{B^2}{(1+\xi\Lambda^2)^2}+\frac{\kappa \gamma}{(1+(\xi\Lambda^2)^{-1})^2}],
\]
where $\gamma,\xi$ are strictly positive scalars that are also functions of $\Lambda$.

\subsection{Non-asymptotic Analysis (for simpler insights)} \label{s:non asym}
We apply the discussion iin Sec. \ref{s:asym} non-asymptotically in few-shot learning. Remember we define $\X\in \R^{\nfs \times R}, \y\in\R^{\nfs}$, each row of $\X$ is independently drawn from $\Nn(0, \bSi)$. We study the following least norm solution of the least squares problem
\begin{align}
    \hat \bt = \arg\min_{\bt'}\ \|\bt'\|,\quad \mbox{s.t.,}~\X\bt' = \y.
\end{align}

\begin{definition}[Non-asymptotic distribution characterization] \label{aux_def}

Set $\kappa=R/{\nfs}>1$. Given $\sigmaze>0$, covariance $\bSi$ and latent vector $\bt$ and define the unique non-negative terms $\xi,\gamma,\z\in\R^R$ and $\bg\in\R^R$ as follows:
\begin{align}
&\xi>0\quad\text{is the solution of}\quad \kappa^{-1}={R^{-1}}\sum_{i=1}^R\big({1+(\xi\bSii)^{-1}}\big)^{-1},\nn\\
&\gamma=\frac{\sigmaze^2+\frac{1}{R}\sum_{i=1}^R\frac{\bSii\bt_i^2}{(1+\xi\bSi)^2}}{1-\frac{\kappa}{R}\sum_{i=1}^R{(1+(\xi\bSii)^{-1})^{-2}}}\nn.
\end{align}
Let $\h\sim\Nn(0,\Id/R)$. The non-asymptotic distributional prediction is given by the following random vector
\[
{\hat \bt}(\bSi)= \frac{1}{1+(\xi\bSi)^{-1}}\odot\bt+\frac{\sqrt{\kappa\gamma}\sbSiinv}{1+(\xi\bSi)^{-1}}\odot \h. 
\]
\end{definition}
Note that, the above formulas can be slightly simplified to have a cleaner look by introducing an additional variable $\z=\frac{1}{1+(\xi\bSi)^{-1}}$.

Also note that, the terms in the non-asymptotic distribution characterization and asymptotic distribution characterization have one to one correspondence. Non-asymptotic distribution characterization is essentially a discretized version of asymptotic DC where instead of expectations (which is integral over pdf) we have summations.

Now, we can use this distribution to predict the test risk by using Def. \ref{aux_def} in the risk expression.

Going back to representation question, without losing generality, assume $\bSi=\Iden$ and let us find optimal $\hbLa_R$. Then
\[
{\hat \bt}=\hbLa_R\left[\frac{1}{1+(\xi\hbLa_R^2)^{-1}}\odot\hbLa_R^{-1}\bt+\frac{\sqrt{\kappa\gamma}\hbLa_R^{-1}}{1+(\xi\hbLa_R^2)^{-1}}\odot \h\right]. 
\]
The risk is given by (using $\h\sim\Nn(0,\Iden_p)$)
\begin{align}
\text{risk}^{\hbLa_R}({\hat \bt}_n) - \sigmaze^2
&=\E[({\hat \bt}-\bt)^\top \bSi ({\hat \bt}-\bt)]\\
&=\sum_{i=1}^R \frac{\Bi}{(1+\xi(\hbLaRi)^2)^2}+\sum_{i=1}^R \frac{\kappa\gamma}{(1+(\xi(\hbLaRi)^2)^{-1})^2}.
\end{align}
Here, note that $\xi$ is function of $\bLa$ and $\gamma$ is function of $\bt,\bLa$. If we don't know $\B$, we use the estimation from representation learning $\hB$ instead. 

To find the optimal representation, we will solve the following optimization problem that minimizes the risk.

\begin{equation}\label{eq: thm1 obj f}
\begin{array}{rrclcl}
\displaystyle \min_{\bLa} & \multicolumn{3}{l}{\sum_{i=1}^R \dfrac{\bt_i^2}{(1+\xi(\hbLaRi)^2)^2}+\sum_{i=1}^R\dfrac{\kappa\gamma}{(1+(\xi(\hbLaRi)^2)^{-1})^{2}}}\\
\\
\textrm{s.t.} & \kappa^{-1}=\dfrac{1}{R}\sum_{i=1}^R (1+(\xi(\hbLaRi)^2)^{-1})^{-1}\\
\\
& \gamma=\dfrac{\sigmaze^2+\sum_{i=1}^R \frac{\bt_i^2}{(1+\xi(\hbLaRi)^2)^2}}{1-\frac{\kappa}{R}\sum_{i=1}^R (1+(\xi(\hbLaRi)^2)^{-1})^{-2}}.
\end{array}
\end{equation}
So we plug in the expression of $\gamma$ and get
\begin{align}
    \kappa\gamma&=\dfrac{\sigmaze^2+\frac{1}{R}\sum_{i=1}^R \frac{\bt_i^2}{(1+\xi(\hbLaRi)^2)^2}}{\kappa^{-1}-\frac{1}{R}\sum_{i=1}^R (1+(\xi(\hbLaRi)^2)^{-1})^{-2}}
    = \dfrac{R\sigmaze^2+\sum_{i=1}^R \frac{\bt_i^2}{(1+\xi(\hbLaRi)^2)^2}}{\sum \frac{\xi(\hbLaRi)^2}{(1 + \xi(\hbLaRi)^2)^2} }.
\end{align}

Let $\btheta_i=\frac{\xi(\hbLaRi)^2}{1+\xi(\hbLaRi)^2}$, then the objective function becomes 
\[
\sum_{i=1}^R\Bi(1-\btheta_i)^2+(\sum_{i=1}^R \btheta_i^2)\dfrac{R\sigmaze^2+\sum \Bi(1-\btheta_i)^2}{\sum_{i=1}^R \btheta_i(1-\btheta_i)} = \dfrac{\nfs(\sum_{i=1}^R\Bi(1-\btheta_i)^2)+R\sigmaze^2(\sum_{i=1}^R \btheta_i^2)}{{\nfs}-\sum_{i=1}^R\btheta_i^2}
\] such that $0\leq\btheta_i<1$ and $\sum_{i=1}^R \btheta_i = \frac{R}{\kappa}={\nfs}$. This quantity is same as the objective \eqref{eq: thm1 obj f}. We divide this quantity by $d$ to get the risk function, which is same as the definition of $f$ in \eqref{optimal_lambda}.

\subsection{Solving the optimization problem.}\label{s:kkt cond app}
Here, we propose the algorithm for minimizing $f(\btheta)$. We explore the KKT condition for its optimality.

The objective function is
\begin{align}
    f(\btheta) = \sum_{i=1}^R\Bi(1-\btheta_i)^2+(\sum_{i=1}^R \btheta_i^2)\dfrac{R\sigmaze^2+\sum \Bi(1-\btheta_i)^2}{\sum_{i=1}^R \btheta_i(1-\btheta_i)}. \label{eq:risk_asy}
\end{align}
\begin{lemma}
Let $C,S,V\in\R$. Define 
\begin{align*}
    \phi(\Bi;C,V,S) := \frac{Cp(R-{\nfs}-S)^2}{2\nfs(V + R\sigmaze^2 + (R-{\nfs}-S)\Bi^2)}
\end{align*}
  and we find the root of the following equations:
 \begin{align*}
    \sum_{i=1}^R \phi(\Bi;C,V,S) &= R-{\nfs},\\
    \sum_{i=1}^R \phi^2(\Bi;C,V,S) &= S - (2\nfs-R),\\
    \sum_{i=1}^R \Bi\phi^2(\Bi;C,V,S) &= V.
\end{align*}
Let $\btheta_i  = 1 - \phi(\Bi;C^*,V^*,S^*)$ where $C^*,V^*,S^*$ are the roots, then 
\begin{align*}
    \btheta = \arg\min_{\btheta'}\ f(\btheta'),\quad \mbox{s.t.,}\ 0\le \btheta'<1,\ \sum_{i=1}^R \btheta'_i = \nfs.
\end{align*}
\end{lemma}
\begin{proof}
Define $s = \sum_{i=1}^R \btheta_i^2$, $\phi_i = 1-\btheta_i$.
Define $Q = \dfrac{1}{R}\sum_{i=1}^R\Bi\phi_i^2$.
Then 
\begin{align*}
    f(\phi) &= \sum_{i=1}^R\Bi\phi_i^2 + \frac{s}{{\nfs}-s} (R\sigmaze^2 + \sum_{i=1}^R\Bi\phi_i^2)\\
    &= R(Q + \frac{s}{{\nfs}-s}(\sigmaze^2 + Q))\\
    &= \frac{{R\nfs}}{R-{\nfs}-\sum_{i=1}^R\phi_i^2} (Q + \sigmaze^2).
\end{align*}
The last line uses 
\begin{align*}
    s &= \sum_{i=1}^R (1 - \phi^2) = R - 2\sum_{i=1}^R\phi_i + \sum_{i=1}^R\phi_i^2= R - 2(R-{\nfs}) + \sum_{i=1}^R\phi_i^2 = 2\nfs - R + \sum_{i=1}^R\phi_i^2.
\end{align*}
Now define $\sum_{i=1}^R\phi_i^2 = S$, and we compute the gradient of $f$, we have 
\begin{align*}
    \frac{df}{R\phi_i} &= \left( 2\nfs(\sum_{j=1}^R \B_j \phi_j^2 + (R-{\nfs}-s)\Bi )+  2R\nfs\sigmaze^2 \right)\phi_i.
\end{align*}
Suppose $0<\phi_i<1$, then we need $ \frac{df}{R\phi_i}$ equal to each other for all $i$. Suppose $ \frac{df}{R\phi_i} = C$, and denote $\sum \B_j \phi_j^2 = V$, we can solve for $\phi_i$ from $ \frac{df}{R\phi_i} = C$ as 
\begin{align}
    \phi_i = \frac{Cd(R-{\nfs}-S)^2}{2\nfs(V + R\sigmaze^2 + (R-{\nfs}-S)\Bi^2)} := \phi(\Bi;C,V,S).\label{eq:phi}
\end{align}
We define the function $\phi(\Bi;C,V,S)$ as above, and use the fact that 
\begin{align*}
    \sum_{i=1}^R \phi(\Bi;C,V,S) &= R-{\nfs},\\
    \sum_{i=1}^R \phi^2(\Bi;C,V,S) &= S - (2\nfs-R),\\
    \sum_{i=1}^R \Bi\phi^2(\Bi;C,V,S) &= V.
\end{align*}
We can solve\footnote{For the root of $3$-dim problem, the worst case we can grid the space and search with time complexity $\Order(\eps^{-3})$.} $C,V,S$ and retrieve $\phi_i$ by \eqref{eq:phi}. $\btheta_i = 1-\phi_i$.
\end{proof}

\section{Analysis of MoM estimators}\label{s:rep app}

\subsection{Covariance estimator}\label{s:cov}

We will first present the estimation error of the feature covariance $\bSi$, which is not covered in the main paper due to limitation of space. Note that if $\bSi$ is fully aligned with $\B$, e.g., $\bSi = \B$, then estimating $\bSi$ is enough for getting optimal representation, and we will show it has lower sample complexity and error compared to estimating canonical covariance $\Bbar$. That is a naive case, if it does not work, this intermediate result will help in our latter proof.

We will use the following Bernstein type concentration lemma, generalized from \cite[Lemma 29]{tripuraneni2020provable}:
\begin{lemma} \label{lem:bern}
Let $\Zb\in\R^{n_1\times n_2}$. Choose $T_0, \sigma^2$ such that 
\begin{enumerate}
    \item $\mtx{P}(\|\Zb\|\ge C_0T_0 + t) \le \exp(-c\sqrt{t/T_0})$.
    \item $\|\Eb(\Zb\Zb^\top)\|, \|\Eb(\Zb^\top \Zb)\|\le \sigma^2$.
\end{enumerate}
Then with probability at least $1-(n T_0)^{-c}$, $c>10$,
\begin{align*}
    \|\frac{1}{n}\sum_{i=1}^n \Zb_i - \Eb(\Zb_i)\| \lesssim \log(n T_0)\left(\frac{T_0\log(n T_0)}{n} + \frac{\sigma}{\sqrt{n}}\right).
\end{align*}
\end{lemma}
\begin{proof}
Define $K = \log^2(C_KnT_0)$ for $C_K>0$, $\Zb' = \Zb\ones(\|\Zb\| \le KT_0)$, then 
\begin{align*}
    \|\mtx{E}(\Zb - \Zb')\| &\le \int_{KT_0}^\infty \exp(-c\sqrt{t/T_0}) dt \lesssim (1+\sqrt{K})\exp(-c\sqrt{K})T_0 \\
    &\lesssim (1+\log(C_KnT_0))(n T_0)^{-C}.
\end{align*}
We can choose $C_K$ large enough so that $C>10$. We will use \cite[Lemma 29]{tripuraneni2020provable}. Set $R = \log^2(C_KnT_0)T_0 + C_0T_0$, $\Delta = (1+\log(C_KnT_0))(n T_0)^{-C}$, $t = C_t \log(n T_0)(\frac{T_0\log(n T_0)}{n} + \frac{\sigma}{\sqrt{n}})$ for some $C_t>0$, plugging in the last inequality of \cite[Lemma 29]{tripuraneni2020provable}, the LHS is smaller than $(n T_0)^{-c}$ for some $c$. We can also check $\mtx{P}(\|\Zb\|\ge R) \le (n T_0)^{-c}$ for some $c$, thus we prove the lemma. 
\end{proof}

\paragraph*{Feature Covariance.} We can directly estimate the covariance of features by
\begin{align}\label{eq:bsi_est}
    \hbSi = \frac{1}{\ntot} \sum_{j=1}^{\nspt}\sum_{i=1}^{\ntask} \xij\xij^\top,
\end{align}
The mean of this estimator is $\bSi$ and we can estimate the top $r$ eigenvector of $\bSi$ with $\tilde \Order(r)$ samples.

As we have defined in Phase 1, features $\xij$ are generated from $\Nn(0,\bSi)$. We aim to estimate the covariance $\bSi$. Although there are different kinds of algorithms, such as maximum likelihood estimator \cite{anderson1970estimation}, to be consistent with the algorithms in the latter sections, we study the sample covariance matrix defined by \eqref{eq:bsi_est}.

\begin{lemma}\label{lem:xxt app}
Suppose $\x_i$, $i=1,...,\ntot$ 
are generated independently from $\Nn(0, \bSi)$. We estimate \eqref{eq:bsi_est}, then when $\ntot\gtrsim\TX$, with probability $1 - \Order((\ntot \tr(\bSi))^{-C})$, 
\begin{align*}
    \|\hbSi - \bSi\| \lesssim \sqrt{\frac{\|\bSi\|\tr(\bSi)}{\ntot}}.
\end{align*}
Denote the span of top $\rf$ eigenvectors of $\bSi$ as $\W$ and the span of top $\rf$ eigenvectors of $\hbSi$ as $\hat \W$. Let $\delta_{\lambda} = \lambda_{\rf}(\bSi) - \lambda_{\rf+1}(\bSi)$. Then if $\ntot \gtrsim \frac{\|\bSi\|\tr(\bSi)}{\delta_{\lambda}^2}$, we have 
\begin{align*}
    \sin(\angle\W,\hat\W) \lesssim \sqrt{\frac{\|\bSi\|\tr(\bSi)}{\ntot\delta_{\lambda}^2}}
\end{align*}
\end{lemma}

\begin{example}
When $\bSi = \diag{\mtx{I}_{\rf}, 0}$, we have $\sin(\angle\W,\hat\W) \lesssim \sqrt{\frac{\rf}{\ntot}}$.
\end{example}

Lemma \ref{lem:xxt app} gives the quality of the estimation of the covariance of features $\x$. When the condition number of the matrix $\bSi$ is close to $1$, we need $\ntot\gtrsim d$ to get an estimation with error $\Order(1)$. However, when the matrix $\bSi$ is close to rank $\barrf$, the amount of samples to achieve the same error is smaller, and we can use $\ntot \gtrsim \barrf$ samples to get $\Order(1)$ estimation error. 

We will use Bernstein type concentration results to bound its error, and a similar technique will be used for $\hat \bM$ in the next sections. 

\begin{proof}
First we observe that, the features $\xij$ among different tasks are generated i.i.d. from $\Nn(0,\bSi)$. So we can rewrite \eqref{eq:bsi_est} as 
\begin{align}\label{eq:bsi_est_2}
    \hbSi = \frac{1}{\ntot }\sum_{i=1}^{\ntot} \x_i\x_i^\top
\end{align}
where $\x_i\sim\Nn(0,\bSi)$. The error of $\hbSi$ depends on $\ntot$ regardless of $\ntask$ and $\nspt$ respectively.

First, we know by concentration inequality 
\begin{align}\label{eq:norm_x_bd}
   \mtx{P}(\|\x\x^\top\| - \tr(\bSi) \ge t) = \mtx{P}(\|\x\|^2 - \tr(\bSi) \ge t) \le \exp(-c\min\{\frac{t^2}{\tr(\bSisq)}, \frac{t}{\|\bSi\|}\}).
\end{align}
We will use the fact $\sqrt{\tr(\bSisq)} \le \tr(\bSi)$. 
Define $K = C_0\log(\ntot \tr(\bSi)) \tr(\bSi)$, $\Zb = \x\x^\top$, $\Zb' = \Zb\cdot\ones\{\|\Zb\|\le K\}$ where $\ones$ means indicator function ($\ones(\mathrm{True}) = 1, \ones(\mathrm{False}) = 0$), for some positive number $C_0$. Then 
\begin{align*}
    \|\Eb(\Zb - \Zb')\| &\le \int_{t=K}^\infty (\exp(-c \frac{t^2}{\tr^2(\bSi)}) + \exp(-c\frac{t}{\|\bSi\|})) dt\\
    & \le \int_{t=K}^\infty (\exp(-c \frac{t}{\tr(\bSi)}) + \exp(-c\frac{t}{\|\bSi\|})) dt\\
    &\le 2\frac{\tr(\bSi)}{c} \exp(-c \frac{K}{\tr(\bSi)})\\
    &\le \frac{\sqrt{K\tr^2(\bSi)}}{c}\exp(-\frac{cK}{\tr(\bSi)})\\
    &\lesssim(\ntot \tr(\bSi))^{-C}
\end{align*}
where $C\ge C_0 - 3/2$.
Then we compute $(\x\x^\top)^2 = \|\x\|^2\x\x^\top$. Let $\bSi$ be diagonal (the proof is invariant from the basis. In other words, if $\bSi$ is not diagonal, then we can make the eigenvectors of $\bSi$ as basis and the proof applies). Then
\begin{align}
    \Eb(\|\x\|^2\x\x^\top)_{ij} =
    \begin{cases}
    \bSi_{ii}(\tr(\bSi) + 2\bSi_{ii}), & i=j,\\
    0, & i\ne j. \label{eq: x pow 4}
    \end{cases}
\end{align}
So $\|\Eb(\|\x\|^2\x\x^\top)\| \le \|\bSi\|(\tr(\bSi) + 2\|\bSi\|) \approx \|\bSi\|\tr(\bSi)$. $\approx$ means $\gtrsim$ and $\lesssim$.

Using  Lemma \ref{lem:bern}, with \eqref{eq:norm_x_bd} and the inequality above, we get that with probability $1 - \Order((\ntot \tr(\bSi))^{-C})$, 
\begin{align}\label{eq:bSi err}
     \|\hbSi - \bSi\| \lesssim \log(\ntot \tr(\bSi))\left(\frac{\log(\ntot \tr(\bSi))\tr(\bSi)}{\ntot } + \sqrt{\frac{\|\bSi\|\tr(\bSi)}{\ntot }} \right).
\end{align}

If the number above is smaller than $\lambda_{r} - \lambda_{r+1}$, we have that
\begin{align}\label{eq:xxt}
    \ntot  \gtrsim \frac{\|\bSi\|\tr(\bSi)}{(\lambda_{r} - \lambda_{r+1})^2}
\end{align}
which is $\Order(r)$ if condition number is $1$.

The bound of the angle of top $R$ eigenvector subspace is a direct application of the following lemma.
\begin{lemma}\label{lem:dk}
\cite{davis1970rotation} Let $\A$ be a square matrix. Let $\hat\W$, $\W$ denote the span of top $r$ singular vectors of $\hat A$ and $\A$. Suppose $\|\hat\A - \A\| \le \Delta$, and $\sigma_r(\A) - \sigma_{r+1}(\A) \ge \Delta$, then
\begin{align*}
    \sin(\angle \W, \hat \W) \le \frac{\Delta}{\sigma_r(\A) - \sigma_{r+1}(\A) - \Delta}.
\end{align*}
\end{lemma}
So that the error of principle subspace recovery of feature covariance is upper bounded by $\frac{\|\hbSi - \bSi\|}{\sigma_{r}(\bSi) - \sigma_{r+1}(\bSi) - \|\hbSi - \bSi\|}$, where $\|\hbSi - \bSi\|$ is calculated in \eqref{eq:bSi err}.
\end{proof}
\vspace{-1em}
\subsection{Method of moment}\label{s:mom app}
This section contains three parts. We first bound the norm of task vectors. Then we analyze the second result of Thm. \ref{thm:mom 1}, where $\nspt$ is lower bounded by effective rank. Last we prove the first result of Thm. \ref{thm:mom 1} which is a generalization of \cite{tripuraneni2020provable}.
\subsubsection{Property of task vectors}
We first study the property of the tasks $\bt_1,...,\bt_{\ntask}$. We know that, for any $\bt\sim\Nn(0,\B)$,
\begin{align*}
    \mtx{P}(\|\bt\|^2 - \tr(\B)\ge t) \le \exp(-c\min\{\frac{t^2}{\tr(\Bsq)}, \frac{t}{\|\B\|}\}).
\end{align*}
So that with probability at least $1-\delta$, we have 
\begin{align}
    \|\bt_i\|^2 &\lesssim \tr(\B) + \sqrt{ (\log(1/\delta) + \log(\ntask))\tr(\Bsq)} + (\log(1/\delta) + \log(\ntask))\|\B\|\notag\\
    &\lesssim \tr(\B) + \log(\ntask/\delta)\sqrt{\tr(\Bsq)}\lesssim \tr(\B)\log(\ntask/\delta),\ \forall i=1,...,\ntask.
    \label{eq: bt i norm}
\end{align}
With similar technique we know that with probability at least $1-\delta$,
\begin{align}
    \|\bSi\bt_i\|^2 &\lesssim \tr(\bSi\B\bSi) + \log(\ntask/\delta)\sqrt{\tr((\bSi\B\bSi)^2)},\ \forall i=1,...,\ntask. \label{eq: bSibt}\\
    \|\sbSi\bt_i\|^2 &\lesssim \tr(\sbSi\B\sbSi) + \log(\ntask/\delta)\sqrt{\tr((\sbSi\B\sbSi)^2)},\ \forall i=1,...,\ntask. \label{eq: sqrbSibt}
\end{align}

We will use $\delta = \ntask^{-c}$ for some constant $c$ so that $\log(\ntask/\delta) = (c+1)\log(\ntask) \approx \log(\ntask)$. Later, we will use the norm bounds of above quantities which happen with probability at least $1 - T^{-c}$.

\subsubsection{Estimating with fewer samples when each task contains enough samples}\label{s:mom ta}

In this part we will prove Theorem \ref{thm:dr app}, which is the second case of Theorem \ref{thm:mom 1}. First we will give a description of standard normal features, then prove the general version.
\begin{theorem}\label{cor:dr app}(Standard normal feature, noiseless)
Let data be generated as in Phase 1, let $\Snorm = \max\{\|\bSi\|, \|\B\|\}$ in this theorem and the following section\footnote{in the paper we assume $\Snorm=1$ for simplicity.}, $\TbtX = \tr(\B\bSi)$,
$\TX = \tr(\bSi)$, $\Tbt = \tr(\B)$. Suppose $\sigma = 0$, $\bSi = \Ib$, and suppose the rank of $\B$ is $\rt$.
Define $\hat \bt_i = \nspt^{-1}\sum_{j=1}^{\nspt}\yij\xij$,  $\mfkb = [\bt_1,...,\bt_{\ntask}]$, and $\hat \mfkb = [\hat \bt_1,...,\hat \bt_{\ntask}]$. Let $\nspt>c_1\Tbt\lambda^{-1}_{\rt}(\B)$,
with probability $1 - \Order(\ntask^{-C})$, where $C$ is constant, 
\begin{align*}
    \sigma_{\max}(\hat \mfkb - \mfkb) \lesssim \sqrt{\frac{\ntask \Tbt}{\nspt}}.
\end{align*}

Denote the span of top $\rt$ singular column vectors of $\hat \mfkb$ and $\B$ as $\hat \W, \W$, then
\begin{align*}
    \sin(\angle\hat \W, \W) \lesssim \sqrt{\frac{\Tbt}{\nspt\lambda_{\rt}(\B)}}.
\end{align*}
\end{theorem}

For example, if $\B = \diag{\Ib_{\rt},0}$, then $\sin(\angle\hat \W, \W) \lesssim \sqrt{\rt/\nspt}$.

\begin{proof} 
We first estimate $\bt_i$ with
\begin{align*}
    \hat \bt_i = \frac{1}{\nspt}\sum_{j=1}^{\nspt} \yij\xij.
\end{align*}
Then we fix $\bt_i$ and compute the covariance of $\yij\xij$ (its mean is $\bt_i$).
\begin{align*}
    \mathrm{Cov}(\yij\xij - \bt_i) &= \Eb(\xij\xij^\top\bt_i\bt_i^\top\xij\xij^\top) - \bt_i\bt_i^\top \precsim \|\bt_i\|^2\Ib.
\end{align*}
The first term is similar to \eqref{eq: x pow 4}, where the bound can is in \cite[Lemma 5]{tripuraneni2020provable}.  The vector $\hat \bt_i$ is the average of $\yij\xij$ over all $j$. With concentration we know that 
\begin{align}
    \mathrm{Cov}(\hat \bt_i - \bt_i) \precsim \frac{\|\bt_i\|^2}{\nspt}\Ib. \label{eq:cov_mfkb}
\end{align}
Let $\mfkb = [\bt_1,...,\bt_{\ntask}]$, and $\hat \mfkb = [\hat \bt_1,...,\hat \bt_{\ntask}]$. Then we know the covariance of each column of $\hat \mfkb - \mfkb$ is bounded by \eqref{eq:cov_mfkb}. Thus with a constant $c$ and probability $1 - \exp(-c\ntask^2)$,
\begin{align}
    \sigma_{\max}^2(\hat \mfkb - \mfkb) \lesssim \frac{\ntask\|\bt_i\|^2}{\nspt}.\label{eq:max_mfkb}
\end{align}

We have proved in \eqref{eq: bt i norm} that $\|\bt_i\|^2 \le \log(\ntask)\tr(\B)$ with probability $1-\ntask^{-c}$. The columns of $\mfkb$ is generated from $\Nn(0,\B)$, so that 
\begin{align*}
    \sigma_{\max}(\hat \mfkb - \mfkb) \lesssim \sqrt{\frac{\ntask\log(\ntask)\tr(\B)}{\nspt}}.
\end{align*}
Now we study $\mfkb$. We know that $\Eb(\mfkb\mfkb^\top) = \Eb(\sum_{i=1}^{\ntask} \bt_i\bt_i^\top) = \ntask\B$. $\mfkb$ is a matrix with independent columns. Thus let $\nspt>c_1\tr(\B)\lambda^{-1}_{\rt}(\B)$,  $\ntask>\max\{c_2d, \frac{\|\B\|\tr(\B)}{\lambda_{\rt}^2(\B)}\}$, then with Lemma \ref{lem:xxt app}, for Gaussian matrix with independent columns \cite{vershynin2010introduction}, with probability at least $1 - \Order(\ntask^{-c_3} + (\ntask\tr(\B))^{-c_4} + \exp(-c_5\ntask^2)) = 1 - \Order(\ntask^{-C})$, where $c_i$ are constants,
\begin{align*}
    \sigma_{\rt}(\mfkb) \ge \sqrt{\ntask\lambda_{\rt}(\B) - \Order(\sqrt{\ntask\|\B\|\tr(\B)})}.
\end{align*}
Denote the span of top $\rt$ singular vectors of $\hat \mfkb$ and $\B$ as $\hat \W, \W$,
with Lemma \ref{lem:dk},  
\begin{align*}
    \sin(\angle\hat \W, \W) \le \sqrt{\frac{\log(\ntask)\tr(\B)}{\nspt\lambda_{\rt}(\B)}}.
\end{align*}
\end{proof}

Next, we will propose a theorem with general feature covariance and noisy data, which is a generalization of Theorem \ref{cor:dr app}.
\begin{theorem}\label{thm:dr app}
Let data be generated as in Phase 1.
Suppose $\hat \bb_i = \nspt^{-1}\sum_{j=1}^{\nspt}\yij\xij$, $\mfkb = \bSi[\bt_1,...,\bt_{\ntask}]$, and $\hat \mfkb = [\hat \bb_1,...,\hat \bb_{\ntask}]$. Let $\delta_\lambda = \lambda_{\rt}(\bSi\B\bSi) - \lambda_{{\rt}+1}(\bSi\B\bSi))$, suppose $\bSi$ is approximately rank $\rf$,
\begin{align*}
    \nspt &\gtrsim (\tr(\B\bSi)+\sigma^2)\|\bSi\|,\\
     \ntask &\gtrsim \max\{\rf, \frac{d\lambda_{\rf+1}(\bSi)}{\|\bSi\|} \}, 
\end{align*}
then with probability $1 - \Order(\ntask^{-C})$, where $C$ is constant, 
\begin{align*}
    \sigma_{\max}(\hat \mfkb - \mfkb) \lesssim \sqrt{\frac{\ntask(\tr(\B\bSi)+\sigma^2)\|\bSi\|}{\nspt}}.
\end{align*}
Denote the span of top ${\rt}$ singular vectors of $\hat \mfkb$ and $\bSi\B\bSi$ as $\hat \W, \W$, if further we assume $\ntask \gtrsim \frac{\|\bSi\B\bSi\|\tr(\bSi\B\bSi)}{\delta^2_\lambda}$, then
\begin{align*}
    \sin(\angle\hat \W, \W) \lesssim \sqrt{\frac{(\tr(\B\bSi)+\sigma^2)\|\bSi\|}{\nspt\delta^2_{\lambda}}}. 
\end{align*}
\end{theorem}
\begin{example}
Suppose $\bSi = \diag{\Ib_{\rf},\iota\Ib_{d-\rf}}$, and $\B = \diag{\Ib_{\rt}, 0}$, $\sigma = 0$. Suppose $\iota d < \rf$. Then with $\ntask \gtrsim \rf$, $\nspt \gtrsim \rt$ so that $\ntot \gtrsim \rf\rt$,
\begin{align*}
     \sin(\angle\hat \W, \W) \lesssim \sqrt{\rt/n}.
\end{align*}
\end{example}

\begin{proof}
We let $\xij\sim\Nn(0, \bSi)$.
For the $i$th task, let 
\begin{align*}
    \hat \bb_i = \frac{1}{\nspt}\sum_{j=1}^{\nspt} \yij\xij.
\end{align*}
We fix $\bt_i$ and compute 
\begin{align}
    \mtx{E}(\yij\xij) &\precsim \Eb(\xij\xij^\top\bt_i) = \bSi\bt_i, \label{eq:hbb mean}
\end{align}
and 
\begin{align}
    \mathrm{Cov}(\yij\xij - \bSi\bt_i) &\precsim (\bt_i^\top\bSi\bt_i)\bSi + \sigma^2\bSi. \label{eq: cov xijyij}
\end{align}
To get the bound above, we can adopt the technique in \cite[Lemma 5]{tripuraneni2020provable} such that, write $\xij = \sbSi \z$, and reduce to $\Eb((\z^\top\sbSi\bt_i)^2\sbSi\z\z^\top\sbSi) $. The proof of \cite[Lemma 5]{tripuraneni2020provable} gives the explicit bound of $\|\Eb((\z^\top\vct{\alpha})^2\z\z^\top)\|$ for any $\vct{\alpha}$ that equals above. The vector $\hat \bb_i$ is the average of $\yij\xij$ over all $j=1,...,\nspt$. With concentration we know that 
\begin{align}
    \mathrm{Cov}(\hat \bb_i - \bSi\bt_i) \precsim \frac{\bt_i^\top\bSi\bt_i + \sigma^2}{\nspt}\bSi. \label{eq:cov_mfkb_gen}
\end{align}
Suppose $\mfkb = \bSi[\bt_1,...,\bt_{\ntask}]$, and $\hat \mfkb = [\bb_1,...,\bb_{\ntask}]$. $\hat \mfkb - \mfkb$ is a matrix with independent columns. Suppose $\X$ is approximately rank $\rf$, Let $\Vb_{\rf}\in\R^{d\times d}$ be the projection onto the top-$R$ sigular vector space of $\bSi$ and $\Vb_{\rf^\perp}\in\R^{d\times d}$ be the projection onto the $\rf+1$ to $d$th sigular vector space of $\bSi$.
With $\ntask$ columns and $\ntask\ge \rf$, we know that
\begin{align*}
    \sigma_{\max}(\Vb_{\rf}(\hat \mfkb - \mfkb)) &\lesssim  \frac{\ntask(\max_i\bt_i^\top\bSi\bt_i + \sigma^2)\|\bSi\|}{\nspt}\\
    \sigma_{\max}(\Vb_{\rf^\perp}(\hat \mfkb - \mfkb)) &\lesssim \frac{\max\{\ntask,d\}(\max_i\bt_i^\top\bSi\bt_i + \sigma^2)\lambda_{{\rt}+1}(\bSi)}{\nspt}
\end{align*}

With similar argument as before, with probability $1 - \exp(-c\ntask^2)$ for constant $c$,
\begin{align}
    \sigma_{\max}^2(\hat \mfkb - \mfkb) \lesssim \frac{\colsmfkb(\max_i\bt_i^\top\bSi\bt_i + \sigma^2)\|\bSi\|}{\nspt}.\label{eq:max_mfkb_gen}
\end{align}
We know in \eqref{eq: sqrbSibt} that $\|\sbSi\bt_i\|^2 \le \Order(\log(\ntask)\tr(\B\bSi))$ with probability $1 - \ntask^{-c}$ for constant $c$. So that 
\begin{align}
    \sigma_{\max}(\hat \mfkb - \mfkb) \lesssim \sqrt{\frac{\colsmfkb(\log(\ntask)\tr(\B\bSi)+\sigma^2)\|\bSi\|}{\nspt}}. \label{eq:mfkb err}
\end{align}
Now we study $\mfkb$. $\Eb(\mfkb\mfkb^\top) = \Eb(\bSi(\sum_{i=1}^{\ntask} \bt_i\bt_i^\top)\bSi) = \ntask\bSi\B\bSi$. 

Thus let
\begin{align*}
    \nspt &> C_1(\log(\ntask)\tr(\B\bSi)+\sigma^2)\|\bSi\|.
\end{align*}
Now apply the concentration of Gaussian matrix with independent columns \cite{vershynin2010introduction}. With probability $1 - \Order(\ntask^{-C_1} + (\ntask\tr(\bSi\B\bSi))^{-C_2} + \exp(-C_3\ntask^2))$, where $C_i$ are constants (the probability can be simplified as $1 - \Order(\ntask^{-C})$), 
\begin{align*}
    \sigma_{\rt}(\mfkb) \ge \sqrt{\ntask(\lambda_{{\rt}}(\bSi\B\bSi) - \lambda_{{\rt}+1}(\bSi\B\bSi)) - \Order(\sqrt{\ntask\|\bSi\B\bSi\|\tr(\bSi\B\bSi)})}.
\end{align*}
Denote the span of top $\rt$ singular vectors of $\hat \mfkb$ and $\bSi\B\bSi$ as $\hat \W, \W$, let 
\begin{align}
    \ntask \gtrsim \max\{\rf, \frac{d\lambda_{\rf+1}(\bSi)}{\|\bSi\|},\frac{\|\bSi\B\bSi\|\tr(\bSi\B\bSi)}{(\lambda_{{\rt}}(\bSi\B\bSi) - \lambda_{{\rt}+1}(\bSi\B\bSi))^2} \} 
\end{align}
we plug in \eqref{eq:mfkb err} and Lemma \ref{lem:dk},
\begin{align*}
    \sin(\angle\hat \W, \W) &\lesssim \sqrt{(\frac{d\lambda_{\rf+1}(\bSi)}{\ntask\|\bSi\|} + 1) \cdot \frac{(\tr(\B\bSi)+\sigma^2)\|\bSi\|}{\nspt(\lambda_{\rt}(\bSi\B\bSi) - \lambda_{{\rt}+1}(\bSi\B\bSi))}}\\
    &\approx \sqrt{\frac{(\tr(\B\bSi)+\sigma^2)\|\bSi\|}{\nspt(\lambda_{\rt}(\bSi\B\bSi) - \lambda_{{\rt}+1}(\bSi\B\bSi))}}. 
\end{align*}
\end{proof}

\subsubsection{Method of moments with arbitrary $\nspt$}
In this subsection we will analyze $\hat\mfkb$ with any $\nspt$, and propose the error of MoM estimator.

First, suppose there are at least two samples per task, we can separate the samples into two halves, and compute the following estimator. 

\begin{theorem}
Let data be generated as in Phase 1, and let $\nspt$ be a even number. Define $\hat \bb_{i,1} = 2\nspt^{-1}\sum_{j=1}^{\nspt/2}\yij\xij$, $\hat \bb_{i,2} = 2\nspt^{-1}\sum_{j=\nspt/2+1}^{\nspt}\yij\xij$. Define 
\begin{align*}
    \hat \bM &= \nspt^{-1}\sum_{i=1}^{\ntask} (\bb_{i,1}\bb_{i,2}^\top + \bb_{i,2}\bb_{i,1}^\top),\\
    \bM &= \bSi\B\bSi.
\end{align*}
Then there is a constant $c>10$, with probability $1 - \ntot^{-c}$,
\begin{align*}
   \| \hat \bM - \bM \| \lesssim (\TbtX + \sigma^2)\sqrt{\frac{\TX}{\ntot}} + \sqrt{\frac{\Tbt}{\ntask}}.
\end{align*}
\end{theorem}

\begin{proof}
For simplicity of notation, we will define a random vector $\x$ with zero mean and covariance $\bSi$, a random vector $\bt$ with zero mean and covariance $\B$, a random variable $\eps$ with zero mean and covariance $\sigma$, and they are subGaussian\footnote{We remove the subscripts when there is no confusion.}. Let $y = \x^\top\bt + \eps$.
We first estimate the mean of $\hat \bM$. 

Note that if we fix $\bt$, $\hat \bb_{i,1}, \hat \bb_{i,2}$ are i.i.d., so
\begin{align*}
    \Eb_{\x,\eps}(\hat \bb_{i,1}) &= \Eb_{\x,\eps}(y\x)
    = \Eb_{\x,\eps}((\x^\top\bt + \eps)\x)= \bSi\bt,\\
    \Eb_{\x,\eps}(\hat \bM) &= \frac{1}{2} (\Eb_{\x,\eps}(\hat \bb_{i,1}) \Eb_{\x,\eps}(\hat \bb_{i,2})^\top + \Eb_{\x,\eps}(\hat \bb_{i,2}) \Eb_{\x,\eps}(\hat \bb_{i,1})^\top)
    \\&= \Eb_{\x,\eps}(\hat \bb_{i,1})\Eb_{\x,\eps}(\hat \bb_{i,1})^\top = \frac{1}{\ntask}\bSi(\sum_{i=1}^{\ntask} \bt_i\bt_i^\top) \bSi.
\end{align*}
We take expectation over $\bt_i$ and get $\bM$. We define the right hand side as $\bar \bM$ for the proof below.

Next, we will bound $\|\hat\bM-\bM\|$. 

\cite[Lemma 3]{tripuraneni2020provable} proposes that, with probability $1-\delta$, 
\begin{align*}
    \|\xij\|^2 &\lesssim \log(1/\delta)\tr(\bSi),\\
    (\xij^\top\bt_i)^2 &\lesssim \log(1/\delta) \tr(\bSi\B),\\
    \eps_{ij}^2 &\lesssim \log(1/\delta)\sigma^2.
\end{align*}
If we enumerate $i=1,...,\ntask$ and $j=1,...,\nspt$, there are in total $\ntask\nspt=\ntot$ terms. So we set $\delta = \ntot^{-c+1}$ for a constant $c>1$, then with probability $1-\ntot^{-c}$, for all $i,j$ we have 
\begin{align*}
    \|\yij\xij\| = \|(\xij\bt_i+\eps_{ij})\xij\| \lesssim \log^{3/2}(\ntot) \sqrt{(\tr(\bSi\B) + \sigma^2)\tr(\bSi)}.
\end{align*}
Define $\ddelta_{i,l} = \hat\bb_{i,l} - \bSi\bt_i$ for $l=1,2$ (we will use $l=1$ below, the result for $l=2$ is the same). Note that $\ddelta_i$ is zero mean. With \cite[Prop. 5.1]{kong2020meta}  we have with probability $1 - \ntot^{-c}$,
\begin{align}
    \|\ddelta_{i,1}\| \lesssim \nspt^{-1/2}\log^{5/2}(\ntot) \sqrt{(\tr(\bSi\B) + \sigma^2)\tr(\bSi)} \label{eq:ddelta bd}
\end{align}
Define 
\begin{align*}
    \Zb_i &= \hat\bb_{i,1}\hat\bb_{i,2}^\top - \Eb_{\x,\eps} (\hat\bb_{i,1}\hat\bb_{i,2}^\top)\\
    &= (\bSi\bt_i + \ddelta_{i,1})(\bSi\bt_i + \ddelta_{i,2})^\top - \Eb_{\x,\eps} (\hat\bb_{i,1}\hat\bb_{i,2}^\top)\\
    &= \ddelta_{i,1}(\bSi\bt_i)^\top + \bSi\bt_i\ddelta_{i,2}^\top + \ddelta_{i,1}\ddelta_{i,2}^\top - \Eb_{\x,\eps}(\ddelta_{i,1}\ddelta_{i,2}^\top).
\end{align*}
Then 
\begin{align}
    \|\Eb \Zb_i\Zb_i^\top\| &\le \|\Eb (\bSi\bt_i\ddelta_{i,2}^\top + \ddelta_{i,1}(\bSi\bt_i)^\top)(\bSi\bt_i\ddelta_{i,2}^\top + \ddelta_{i,1}(\bSi\bt_i)^\top)^\top\| \notag\\
    &\quad + \|\Eb \ddelta_{i,1}\ddelta_{i,2}^\top\ddelta_{i,2}\ddelta_{i,1}^\top\|. \label{eq: Zbi pause 2}
\end{align}
Then we can use \eqref{eq:ddelta bd} and \eqref{eq: bSibt} to bound the first term by
\begin{align*}
    \nspt^{-1}\log^{6}(\ntot)(\tr(\bSi\B) + \sigma)\tr(\bSi)\tr(\bSisq\B)\|\bSi\|^2. 
\end{align*}
And
\begin{align*}
    \Eb_{\x,\eps} \ddelta_{i,1}\ddelta_{i,2}^\top\ddelta_{i,2}\ddelta_{i,1}^\top &= (\Eb_{\x} \ddelta_{i,2}^\top\ddelta_{i,2})\|\Eb_{\x} \ddelta_{i,1}\ddelta_{i,1}^\top\|\\
    &\lesssim \nspt^{-2}(\Eb_{\x,\eps} (\x^\top\bt+ \eps)^2 \x^\top\x) \|\Eb_{\x,\eps} (\x^\top\bt+ \eps)^2 \x\x^\top\|\\
    &\lesssim \nspt^{-2}(\tr^2(\bSi\B) + \sigma^4)\tr(\bSi)\|\bSi\|.
\end{align*}
The second line is due to the fact that $\ddelta_{i,l}$ is the difference of $(\x^\top\bt+ \eps)\x$ and its mean, and covariance is upper bounded by variance (not subtracting the mean). The $\nspt^{-2}$ factor comes from the average over $\nspt$ terms. The reasoning of the last line is same 
as \eqref{eq: cov xijyij}. Now we can go back to \eqref{eq: Zbi pause 2} and get
\begin{align*}
    \|\Eb \Zb_i\Zb_i^\top\| \lesssim \nspt^{-1}\log^{6}(\ntot)(\tr(\bSisq\B) + \tr(\bSi\B) + \sigma^2)^2\tr(\bSi)\|\bSi\|^2.
\end{align*}
Next we need to bound the norm of $\Zb_i$. We use \eqref{eq:ddelta bd} and \eqref{eq: bSibt}, with probability $1 - \ntot^{-c}$,  
\begin{align*}
    \|\Zb_i\| &\le \nspt^{-1/2}\log^{3}(\ntot)(\tr(\bSisq\B) + \tr(\bSi\B) + \sigma^2)\sqrt{\tr(\bSi)}\|\bSi\| \\
    &\quad + \nspt^{-1}\log^{5}(\ntot) (\tr(\bSi\B) + \sigma^2)\tr(\bSi).
\end{align*}
Define the upper bound for $\|\Eb \Zb_i\Zb_i^\top\|, \|\Zb_i\|$ as $Z_1, Z_2$ (the right hand side of two above inequalities).
Now we apply Bernstein type inequality (Lemma \ref{lem:bern}), with probability $1 - \ntot^{-c}$, 
\begin{align*}
    &\quad \| \hat \bM - \bar\bM\| \\
    &= 
    \| \ntask^{-1} \sum_{i=1}^{\ntask} \Zb_i - \Eb_{\x} \Zb_i\|\\
    &\lesssim \log(\ntask Z_2) \left( \ntask^{-1/2}\log(\ntot) Z_1^{1/2} + \ntask^{-1}Z_2\log(\ntask Z_2) \right)\\
    &\lesssim \log(\ntask Z_2) \Big(\sqrt{\frac{\log^{6}(\ntot)(\tr(\bSisq\B) + \tr(\bSi\B) + \sigma^2)^2\tr(\bSi)\|\bSi\|^2}{\nspt\ntask}} \\
    &\quad + \frac{\log^{3}(\ntot)(\tr(\bSisq\B) + \tr(\bSi\B) + \sigma^2)\sqrt{\tr(\bSi)}\|\bSi\|}{\nspt^{1/2}\ntask}\\
    &\quad + \frac{\log^{5}(\ntot) (\tr(\bSi\B) + \sigma^2)\tr(\bSi)}{\ntask}\Big)\\
    &= \log(\ntask Z_2)\cdot \Big(\log^{3}(\ntot)\|\bSi\|(\tr(\bSisq\B) + \tr(\bSi\B) + \sigma^2)\sqrt{\frac{\tr(\bSi)}{\ntot}}\\
    &\quad + \frac{\log^{5}(\ntot)(\tr(\bSisq\B) + \tr(\bSi\B) + \sigma^2)\sqrt{\tr(\bSi)}\|\bSi\|}{\ntot^{1/2}\ntask^{1/2}}\Big).
\end{align*}
The term 
\begin{align*}
    \|\bSi\|(\tr(\bSisq\B) + \tr(\bSi\B) + \sigma^2)\sqrt{\frac{\tr(\bSi)}{\ntot}}
\end{align*}
is the dominant term as shown in the theorem.
\end{proof}

The following method of moment estimator is used in \cite{tripuraneni2020provable}, where $\nspt\ge 1$. In other words, if there is one sample per task, one can use the following estimator.
\begin{theorem}
Let data be generated as in Phase 1. Define $\hat \bb_i = \nspt^{-1}\sum_{j=1}^{\nspt}\yij\xij$, $\mfkb = \bSi[\bt_1,...,\bt_{\ntask}]$, and $\hat \mfkb = [\hat \bb_1,...,\hat \bb_{\ntask}]$. 
Define
\begin{align*}
    \hat \BB &= \hat \mfkb \hat \mfkb^\top = \ntask^{-1}\sum_{i=1}^{\ntask} \hat \bb_i\hat \bb_i^\top, \\
    \BB &= \Eb(\hat \mfkb \hat \mfkb^\top) = \bSi\B\bSi + \nspt^{-1}(\bSi\B\bSi + \tr(\B\bSi)\bSi + \sigma^2\bSi ) ,\\
    \barB &= \sum_{i=1}^{\ntask} \bt_i\bt_i^\top,\\
    \bBB &= \bSi\barB\bSi + \nspt^{-1}(\bSi\barB\bSi + \tr(\barB\bSi)\bSi + \sigma^2\bSi )
\end{align*}
With probability $1 - \ntot^c$, 
\begin{align*}
    \|\hat \BB - \bBB\| \lesssim 
    \|\bSi\|(\tr(\bSisq\B) + \tr(\bSi\B) + \sigma^2)\sqrt{\frac{\tr(\bSi)}{\ntot}}.
\end{align*}
\end{theorem}
\begin{proof}
First, we compute the expectation of $\hat \BB$. 
\begin{align}
    \Eb_{\x,y,\eps}\hat \BB &= \Eb_{\x,y,\eps} \ntask^{-1}(\sum_{i=1}^{\ntask} \hat\bb_i\hat\bb_i^\top),\notag\\
    \Eb_{\x,y,\eps}\hat\bb_i\hat\bb_i^\top &= \Eb_{\x,y,\eps} \left(\nspt^{-1} \sum_{j=1}^{\nspt} (\bt_i^\top\xij + \eps_{ij})\xij\right)\left(\nspt^{-1} \sum_{j=1}^{\nspt} (\bt_i^\top\xij + \eps_{ij})\xij\right)^\top\notag \\
    &= \nspt^{-1}\sigma^2\bSi + \Eb_{\x}  (\nspt^{-1}\sum_{j=1}^{\nspt}\xij\xij^\top\bt_i)(\nspt^{-1}\sum_{j=1}^{\nspt}\xij\xij^\top\bt_i)^\top. \label{eq: BBi mean}  
\end{align}
Now we will study the second term. \eqref{eq:hbb mean} states that $\Eb_{\x,y,\eps}(\hat\bb_i) = \bSi\bt_i$. And $\hat\bb_i$ is an average of $\nspt$ terms, we use the expression of the covariance of sample means to get
\begin{align}
    \mathbf{Cov}(\hat\bb_i) &= \nspt^{-1}\mathbf{Cov}(\x\x^\top\bt_i), \label{eq: BBi cov} \\
   \Eb_{\x,y,\eps}\hat\bb_i\hat\bb_i^\top &=
    \Eb_{\x}  (\nspt^{-1}\sum_{j=1}^{\nspt}\xij\xij^\top\bt_i)(\nspt^{-1}\sum_{j=1}^{\nspt}\xij\xij^\top\bt_i)^\top \notag \\
    &= \bSi\bt_i\bt_i^\top\bSi + \nspt^{-1} \mathbf{Cov} (\x\x^\top\bt_i) \label{eq: BBi mean 2} 
\end{align}
Now we study $\mathbf{Cov}(\x\x^\top\bt_i)$. 
\begin{align*}
    \mathbf{Cov}(\x\x^\top\bt_i) &= \Eb_{\x}(\x\x^\top\bt_i - \bSi\bt_i)(\x\x^\top\bt_i - \bSi\bt_i)^\top\\
    &= \Eb_{\x} (\x\x^\top\bt_i)(\x\x^\top\bt_i)^\top - \bSi\bt_i\bt_i^\top\bSi
\end{align*}
Let $\x = \sqrt{\bSi}\z$ so that $\z\sim\Nn(0,\Ib)$. Let two indices $k,l\in[\d]$. When $k\ne l$,
\begin{align*}
    \Eb_{\x} [(\x\x^\top\bt_i)(\x\x^\top\bt_i)^\top]_{kl}
    &= \Eb_{\z}(\sum_{j=1}^d\bt_{i,j}\sigma_j\z_j)^2 \sigma_k\z_k\sigma_l\z_l\\
    &= 2\sigma_k^2\sigma_l^2\bt_{i,k}\bt_{i,l}
\end{align*}
And
\begin{align*}
     \Eb_{\x} [(\x\x^\top\bt_i)(\x\x^\top\bt_i)^\top]_{kk}
     &= \Eb_{\z}(\sum_{j=1}^d\bt_{i,j}\sigma_j\z_j)^2 \sigma_k^2\z_k^2\\
     &= \tr(\bt_i^\top\bSi\bt_i)\sigma_k^2 + 2\sigma_k^4\bt_{i,k}^2.
\end{align*}
So that 
\begin{align*}
\Eb_{\x} (\x\x^\top\bt_i)(\x\x^\top\bt_i)^\top &= 2\bSi\bt_i\bt_i^\top\bSi + \tr(\bt_i^\top\bSi\bt_i),\\
    \mathbf{Cov}(\x\x^\top\bt_i) &= \Eb_{\x} (\x\x^\top\bt_i)(\x\x^\top\bt_i)^\top - \bSi\bt_i\bt_i^\top\bSi \\
    &= \bSi\bt_i\bt_i^\top\bSi + \tr(\bt_i^\top\bSi\bt_i)\bSi.
\end{align*}
We plug it back into \eqref{eq: BBi mean 2} and \eqref{eq: BBi mean} and get
\begin{align*}
    \Eb_{\x,y,\eps}\hat\bb_i\hat\bb_i^\top &= \bSi\bt_i\bt_i^\top\bSi + \nspt^{-1}(\bSi\bt_i\bt_i^\top\bSi + \tr(\bt_i^\top\bSi\bt_i)\bSi + \sigma^2\bSi ).
\end{align*}
Define $\barB = \frac{1}{\ntask}\sum_{j=1}^{\ntask} \bt_j\bt_j^\top$. So that 
\begin{align*}
    \Eb_{\x,y,\eps}\hat \BB &= \Eb_{\x,y,\eps} \ntask^{-1}(\sum_{i=1}^{\ntask} \hat\bb_i\hat\bb_i^\top) \\
    &= \bSi\barB\bSi + \nspt^{-1}(\bSi\barB\bSi + \tr(\barB\bSi)\bSi + \sigma^2\bSi ) := \bBB.\\
    \Eb_{\bt}\hat \BB &= \BB.
\end{align*}
We fix all $\bt_i$ and study $\Eb_{\x,y,\eps}\hat \BB$. Now we need to show how fast $\hat\BB$ converges to $\bBB$.  

Define 
\begin{align*}
    \Zb_i &= \hat\bb_i\hat\bb_i^\top - \Eb_{\x} (\hat\bb_i\hat\bb_i^\top)\\
    &= (\bSi\bt_i + \ddelta_i)(\bSi\bt_i + \ddelta_i)^\top - \Eb_{\x}(\bSi\bt_i + \ddelta_i)(\bSi\bt_i + \ddelta_i)^\top\\
    &= \bSi\bt_i\ddelta_i^\top + \ddelta_i(\bSi\bt_i)^\top + \ddelta_i\ddelta_i^\top - \Eb_{\x}(\bSi\bt_i\ddelta_i^\top + \ddelta_i(\bSi\bt_i)^\top + \ddelta_i\ddelta_i^\top).
\end{align*}
Then 
\begin{align*}
    \|\Eb \Zb_i^2\| &\le \|\Eb (\bSi\bt_i\ddelta_i^\top + \ddelta_i(\bSi\bt_i)^\top)^2\| + \|\Eb \ddelta_i\ddelta_i^\top\ddelta_i\ddelta_i^\top\|.
\end{align*}

Then we can use \eqref{eq:ddelta bd} and \eqref{eq: bSibt} to bound the first term
\begin{align}
    \|\Eb \Zb_i^2\| \lesssim \nspt^{-1}\log^{6}(\ntot)(\tr(\bSi\B) + \sigma)\tr(\bSi)\tr(\bSisq\B)\|\bSi\|^2 + \|\Eb \ddelta_i\ddelta_i^\top\ddelta_i\ddelta_i^\top\| \label{eq: Zbi pause}
\end{align}
So we need to bound $\|\Eb \ddelta_i\ddelta_i^\top\ddelta_i\ddelta_i^\top\|$. Note that $\ddelta_i$ is the average of $\xij(\xij^\top\bt_i + \eps_{ij})$ with respect to index $j=1,...,\nspt$. So we just let $\x\sim \Nn(0,\bSi)$ and study $\x(\x^\top\bt_i+ \eps_{ij})$. Denote it by $\ub_i$.
\begin{align*}
   \|\Eb_{\x} \ub_i\ub_i^\top\ub_i\ub_i^\top\| &= \| \Eb_{\x} (\x^\top\bt_i+ \eps_{ij})^4 \x\x^\top\x\x^\top\|\\
   &\lesssim \|\Eb_{\x} ((\x^\top\bt_i)^4 + \sigma^4)\x\x^\top\x\x^\top\|\\
   &\lesssim (\tr^2(\bSi\B) + \sigma^4)\tr(\bSi)\|\bSi\|. 
\end{align*}
So that 
\begin{align*}
    \|\Eb \ddelta_i\ddelta_i^\top\ddelta_i\ddelta_i^\top\| \lesssim \nspt^{-2}(\tr^2(\bSi\B) + \sigma^4)\tr(\bSi)\|\bSi\|.
\end{align*}
Now we can go back to \eqref{eq: Zbi pause} and get
\begin{align*}
    \|\Eb \Zb_i^2\| \lesssim \nspt^{-1}\log^{6}(\ntot)(\tr(\bSisq\B) + \tr(\bSi\B) + \sigma^2)^2\tr(\bSi)\|\bSi\|^2.
\end{align*}
Next we need to bound the norm of $\Zb_i$. We use \eqref{eq:ddelta bd} and \eqref{eq: bSibt}, with probability $1 - \ntot^{-c}$,  
\begin{align*}
    \|\Zb_i\| \le \nspt^{-1/2}\log^{3}(\ntot)(\tr(\bSisq\B) + \tr(\bSi\B) + \sigma^2)\sqrt{\tr(\bSi)}\|\bSi\| + \nspt^{-1}\log^{5}(\ntot) (\tr(\bSi\B) + \sigma^2)\tr(\bSi).
\end{align*}
Define the upper bound for $\|\Eb \Zb_i^2\|, \|\Zb_i\|$ as $Z_1, Z_2$ (the right hand side of two above inequalities).
With Bernstein type inequality (Lemma \ref{lem:bern}),with probability $1 - \ntot^{-c}$, 
\begin{align*}
    &\quad \| \hat \BB - \bBB\| \\
    &= 
    \| \ntask^{-1} \sum_{i=1}^{\ntask} \Zb_i - \Eb_{\x} \Zb_i\|\\
    &\lesssim \log(\ntask Z_2) \left( \ntask^{-1/2}\log(\ntot) Z_1^{1/2} + \ntask^{-1}Z_2\log(\ntask Z_2) \right)\\
    &\lesssim \log(\ntask Z_2) \Big(\sqrt{\frac{\log^{6}(\ntot)(\tr(\bSisq\B) + \tr(\bSi\B) + \sigma^2)^2\tr(\bSi)\|\bSi\|^2}{\nspt\ntask}} \\
    &\quad + \frac{\log^{3}(\ntot)(\tr(\bSisq\B) + \tr(\bSi\B) + \sigma^2)\sqrt{\tr(\bSi)}\|\bSi\|}{\nspt^{1/2}\ntask}\\
    &\quad + \frac{\log^{5}(\ntot) (\tr(\bSi\B) + \sigma^2)\tr(\bSi)}{\ntask}\Big)\\
    &= \log(\ntask Z_2)\cdot \Big(\log^{3}(\ntot)\|\bSi\|(\tr(\bSisq\B) + \tr(\bSi\B) + \sigma^2)\sqrt{\frac{\tr(\bSi)}{\ntot}}\\
    &\quad + \frac{\log^{5}(\ntot)(\tr(\bSisq\B) + \tr(\bSi\B) + \sigma^2)\sqrt{\tr(\bSi)}\|\bSi\|}{\ntot^{1/2}\ntask^{1/2}}\Big).
\end{align*}
\end{proof}

\section{Proof of Robustness of Optimal Representation}\label{s:robust app}
\setcounter{theorem}{2}
\begin{theorem}
Suppose the data is generated as Phase 2, $\hbLa$ and $\uth$ are defined in Def. \ref{thm:opt rep theta} and the estimated task is obtained as \eqref{def:minl2_interpolator}. Let the upper bound of $\|\hat \bM - \bM\|$ be $\calE$. The risk of meta-learning algorithm satisfies
\vspace{-0.2em}
\begin{align*}
     \risk(\hbLa_{\uth}(R), \B, \bSi) - \risk(\bLa_{\uth}(R),\B, \bSi) \lesssim 
    \frac{\nfs^2\cdot\calE}{d(R-\nfs)(2\nfs - R\uth)\uth}. 
\end{align*}
\end{theorem}
\begin{proof}
In the proof below, we use $\hbLa$ and $\bLa$ to replace $\hbLa_{\uth}(R),\bLa_{\uth}(R)$ for simplicity.
We first decompose the risk as 
\begin{align*}
    &\quad \risk(\hbLa, \B, \bSi) - \risk(\bLa, \B, \bSi)\\
    &= \underbrace{\risk(\hbLa, \hB,\bSi) - \risk(\bLa, \hB,\bSi)}_{\le 0} \\
    &\quad + [\risk(\hbLa, \B, \bSi)-\risk(\hbLa, \hB,\bSi)] + [\risk(\bLa, \hB,\bSi)-\risk(\bLa, \B, \bSi)].
\end{align*}
We know $\risk(\hbLa, \hB,\bSi) - \risk(\bLa, \hB,\bSi)\le 0$ due to the optimality of $\hbLa$ with task covariance $\hB$.
Now we will bound $\risk(\hbLa, \B, \bSi)-\risk(\hbLa, \hB,\bSi)$ for arbitrary $\hbLa$, and it automatically works for $\risk(\bLa, \hB,\bSi)-\risk(\bLa, \B, \bSi)$. Note that in \eqref{optimal_lambda} we know that 
\begin{align}
   &\quad \risk({\hbLa}',\Bp) = f(\btheta;\B, \bSi): = \sum_{i=1}^R\dfrac{\nfs(1 - \btheta_i)^2}{R(\nfs - \|\btheta\|^2)}\BbarRi + \frac{\nfs}{\nfs - \|\btheta\|^2}\sigma^2. \label{eq:risk f app}
\end{align}
This function is linear in $\B$ thus we know that 
\begin{align}
    |\risk(\bLa, \hB,\bSi)-\risk(\bLa, \B, \bSi)| \le \frac{\nfs}{d(\nfs-\|\btheta\|^2)}\calE.\label{eq:risk f app2}
\end{align}
Now we need to bound $\|\btheta\|^2$. With the constraint $\uth \le \btheta<1-\frac{R-\nfs}{\nfs}\uth$ and $\sum\btheta_i=\nfs$, we know that the maximum of $\|\btheta\|^2$ happens when $(R-\nfs)$ among $\btheta_i$ are $\uth$ and the others are $1-\frac{R-\nfs}{\nfs}\uth$. With this we have 
\begin{align*}
    \|\btheta\|^2 &\le (R-\nfs)\uth^2 + \nfs(1-\frac{R-\nfs}{\nfs}\uth)^2\\
    &= (R-\nfs)\uth^2 + \nfs - 2(R-\nfs)\uth + \frac{(R-\nfs)^2}{\nfs}\uth^2 \\
    &= \nfs - 2(R-\nfs)\uth + \frac{(R-\nfs)R}{\nfs}\uth^2
\end{align*}
Thus 
\begin{align*}
   \nfs - \|\btheta\|^2 \ge  (R-\nfs)\uth (2\nfs - R\uth).
\end{align*}
Plugging it into \eqref{eq:risk f app2} and \eqref{eq:risk f app} leads to the theorem. 
\end{proof}

\end{document}